\documentclass[a4paper,fleqn]{cas-dc}

\usepackage[numbers]{natbib}
\usepackage{algorithm}
\usepackage{algorithmic}
\usepackage{amsthm} 
\usepackage{amsmath} 
\usepackage[font=small,labelfont=bf]{caption}

\newtheorem{proposition}{Proposition}
\newtheorem{lemma}{Lemma}

\def\tsc#1{\csdef{#1}{\textsc{\lowercase{#1}}\xspace}}
\tsc{WGM}
\tsc{QE}
\tsc{EP}
\tsc{PMS}
\tsc{BEC}
\tsc{DE}

\begin{document}
\let\WriteBookmarks\relax
\def\floatpagepagefraction{1}
\def\textpagefraction{.001}
\shorttitle{RAGP: Redundancy-Aware Graph Pruning}
\shortauthors{Yaxin Gao et~al.}

\title [mode = title]{Mapping Text to Multiplex Graph: Prompt Compression \\as Lévy Walk-Guided Graph Pruning}                      



\author[1,2]{Yaxin Gao}[style=chinese]
\cormark[1]
\credit{Conceptualization, Methodology, Software, Writing - original draft, Visualization}
\ead{gaoyaxin@zjut.edu.cn}

\author[1,2]{Yao Lu}[style=chinese]
\cormark[1]
\credit{Conceptualization, Methodology, Writing - review & editing}
\ead{yaolu.zjut@gmail.com}

\author[3]{Jinhong Deng}[style=chinese]
\credit{Validation, Formal analysis, Investigation}
\ead{jhdengvision@gmail.com}

\author[1,2]{Jiaqi Nie}[style=chinese]
\credit{Formal analysis, Investigation}
\ead{jiaqinie@zjut.edu.cn}

\author[2]{Zhe Tang}[style=chinese]
\credit{Investigation, Data curation}
\ead{zhe.tang15@student.xjtlu.edu.cn}

\author[2,4]{Jian Zhang}[style=chinese]
\credit{Investigation, Visualization}
\ead{zhangjian-hdu@hdu.edu.cn}

\author[5]{Zhaowei Zhu}[style=chinese]
\ead{zzw@d5data.ai}
\credit{Visualization, Writing - review & editing, Supervision}

\author[1,2]{Shanqing Yu}[style=chinese]
\cormark[2]
\ead{yushanqing@zjut.edu.cn}
\credit{Resources, Writing - review & editing, Supervision}

\author[1,2]{Qi Xuan}[style=chinese]
\ead{xuanqi@zjut.edu.cn}

\author[6,7]{Joey Tianyi Zhou}[style=chinese]
\credit{Resources, Funding acquisition}
\ead{joey.tianyi.zhou@gmail.com}

\affiliation[1]{organization={Institute of Cyberspace Security, Zhejiang University of Technology},
                city={Hangzhou},
                country={China}}
\affiliation[2]{organization={Binjiang Institute of Artificial Intelligence, Zhejiang University of Technology},
                city={Hangzhou},
                country={China}}
\affiliation[3]{organization={University of Electronic Science and Technology of China},
                city={Chengdu},
                country={China}}
\affiliation[4]{organization={School of Cyberspace, Hangzhou Dianzi University},
                city={Hangzhou},
                country={China}}
\affiliation[5]{organization={D5 Data},
                city={Hangzhou},
                country={China}}
\affiliation[6]{organization={Centre for Frontier AI Research, Agency for Science, Technology and Research},
                city={Singapore},
                country={Singapore}}
\affiliation[7]{organization={Institute of High Performance Computing, Agency for Science, Technology and Research},
                city={Singapore},
                country={Singapore}}

\cortext[cor1]{Equal Contribution}
\cortext[cor2]{Corresponding author}

\begin{abstract}
Existing prompt compression methods treat text as flat token sequences,
failing to capture the distributed nature of important information,
which is often spread across multiple locations and connected through
both local syntactic dependencies and global semantic relations.
Such relational structure is naturally represented as a graph,
where tokens or sentences become nodes and their dependencies become edges.
To this end, we propose \textbf{RAGP}, which formulates prompt compression
as \textbf{R}edundancy-\textbf{A}ware \textbf{G}raph \textbf{P}runing on a multiplex graph that jointly models
fine-grained attention-based dependencies and coarse-grained semantic relations.
To efficiently identify non-redundant nodes in this heterogeneous structure
(dense local subgraphs and sparse global connections),
we employ Lévy walks whose heavy-tailed step distribution naturally balances local exploitation with global exploration.
Experiments on LongBench show that RAGP achieves an average score of $49.3$ under a $4\times$ compression ratio, outperforming existing LLM-based compression methods, such as LongLLMLingua, which attains $48.8$ at a $3\times$ compression ratio.
Besides, RAGP also surpasses state-of-the-art vision-based text compression paradigms on multiple tasks.

\end{abstract}

\begin{graphicalabstract}
\centering
\includegraphics[width=\linewidth]{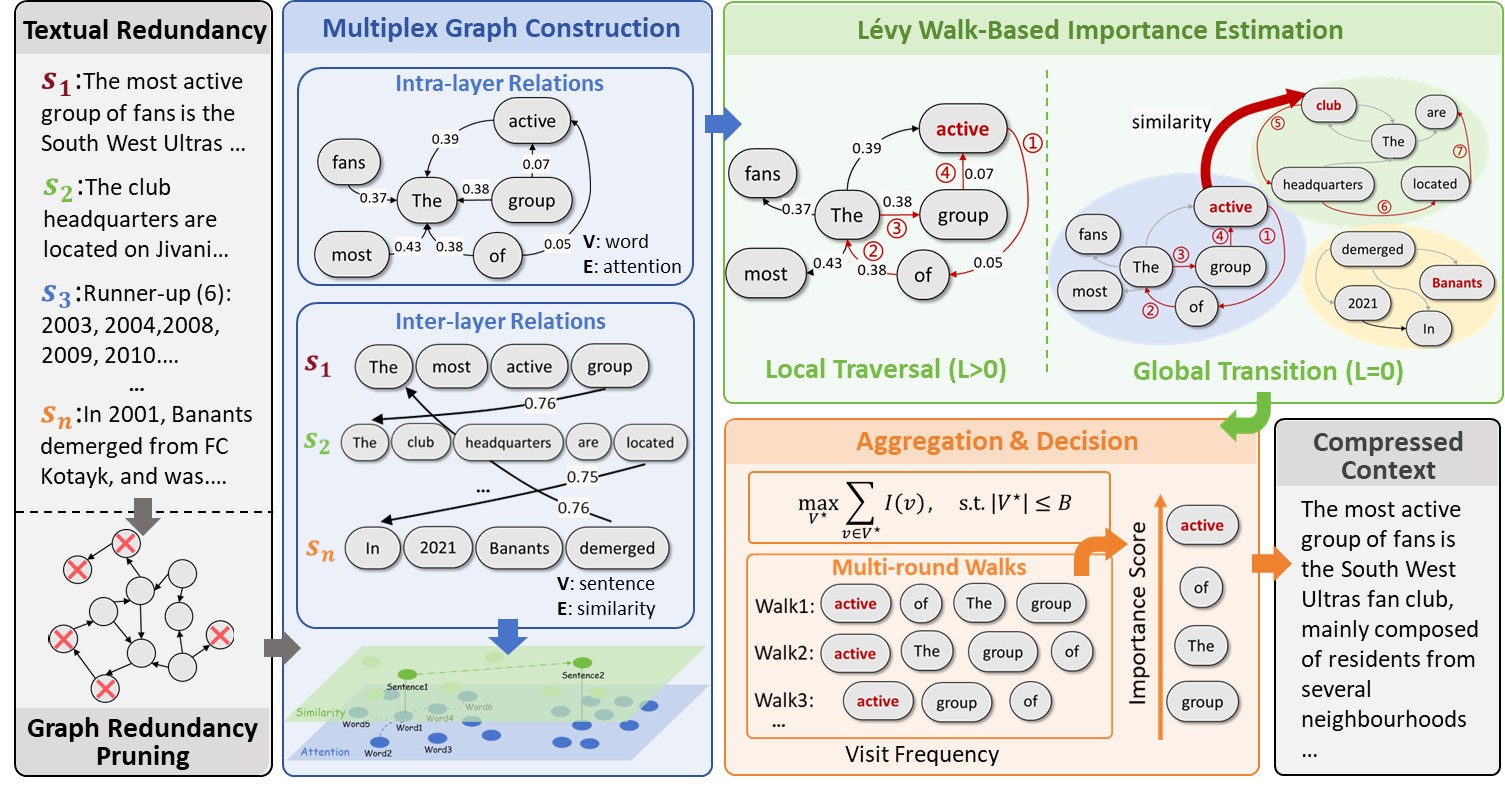}
\end{graphicalabstract}

\begin{highlights}
\item Prompt compression is framed as a redundancy-aware graph pruning problem.
\item Lévy walk-based estimation balances local and global importance exploration.
\item RAGP achieves state-of-the-art results on the LongBench benchmark.
\item Superior performance at 4x compression ratio compared to LLM-based baselines.
\end{highlights}


\maketitle

\section{Introduction}

Large language models (LLMs) have achieved remarkable success across a wide range of natural language processing tasks~\cite{devlin2019bert,izacard2020leveraging,zhang2020pegasus,feng2020codebert,li2022competition}. However, their reliance on long and information-rich prompts introduces several practical limitations, including high inference cost, increased latency, excessive memory consumption, and potential performance degradation due to diluted attention over long contexts~\cite{jiang2023llmlingua,pan2024llmlingua,jiang2023longllmlingua,liu2023lost}. These challenges collectively highlight the importance of prompt compression for enabling efficient, accurate, and scalable LLM inference.

Most existing prompt compression approaches operate at the token level~\cite{chen2025dast,jiang2023longllmlingua,pan2024llmlingua,tang2025perception}, aiming to reduce input length by pruning or reweighting tokens based on attention scores~\cite{fang2025attentionrag,honig2025better} or heuristic importance measures~\cite{wang2025beyond,fu2025deep}. 
While effective in reducing computational overhead, these methods typically treat text as a flat sequence, ignoring the key insight that important information in long documents is \emph{distributed} rather than isolated. Such information is intricately connected through both local syntactic dependencies (within sentences) and global semantic relations (across sentences). By ignoring this structured organization, existing token-level methods~\cite{jiang2023longllmlingua,pan2024llmlingua} often make suboptimal decisions: they may retain locally salient but globally redundant tokens, or conversely, discard tokens that appear individually unremarkable but are collectively essential. This limitation motivates a graph-based formulation, where tokens or sentences are modeled as nodes and their dependencies as edges, naturally enabling structure-aware redundancy pruning.

To achieve this structure awareness, we construct a multiplex graph~\cite{melton2023muxgnn,shen2024beyond} that naturally captures the dual nature of document structure: a fine-grained layer models dense local dependencies via attention, while a coarse-grained layer models sparse long-range semantic connections. This formulation transforms prompt compression into a graph-theoretic problem of identifying and pruning redundant nodes while retaining structurally important ones.

However, the resulting graph exhibits heterogeneous connection patterns (i.e., dense local subgraphs and sparse global links), posing a significant challenge to standard graph traversal methods~\cite{page1999pagerank,xing2004weighted} that tend to remain trapped in local neighborhoods. To address this, we propose RAGP, a novel framework that employs stochastic Lévy walks for redundancy-aware multiplex graph pruning. Characterized by heavy-tailed step lengths, Lévy walks naturally alternate between local exploitation and occasional long-range jumps, effectively navigating this heterogeneous graph structure. By selecting nodes that are frequently visited during this process, RAGP effectively filters out redundancy while preserving the distributed semantic structure, yielding a high-quality compressed prompt. In summary, our contributions can be summarized as follows:
\begin{itemize}
    \item We construct a multiplex graph to integrate fine-grained local dependencies with coarse-grained global semantic relations in text, effectively transforming prompt compression into a \emph{redundancy-aware graph pruning} problem on this heterogeneous structure.
    \item We introduce Lévy walk-based importance estimation, whose heavy-tailed step distribution naturally balances local exploitation within dense subgraphs and global exploration across sparse links, effectively navigating the heterogeneous multiplex structure.
    \item Extensive experiments show that RAGP achieves state-of-the-art performance on LongBench, recording an average score of $49.3$ at a $4\times$ compression ratio and outperforming both competitive LLM-based baselines and vision-based text compression methods.
\end{itemize}

\section{Related Work}

\label{sec:Related Work}

\subsection{Prompt Compression}
Prompt compression methods aim to reduce input length while retaining task-relevant information.
Early approaches estimate the importance of tokens or sentences using simple criteria, such as statistical relevance scores~\cite{li2023unlocking,lin2404rho,tang2025perception}, to prune less informative content in a task-agnostic manner.
Later methods incorporate internal signals from neural language models to guide compression decisions, exploiting model-internal cues including attention distributions~\cite{zhao2025leveraging,chen2025pis,honig2025better} or loss sensitivity~\cite{quancai2025discomp}, either directly for inference-time pruning~\cite{ma2025cot,kang2025c3ot,zhao2025can} or as supervision for lightweight compressors.
More recent frameworks integrate multiple signals within multi-stage pipelines~\cite{jiang2023llmlingua,pan2024llmlingua,jiang2023longllmlingua}, combining coarse-grained filtering with fine-grained refinement.
PartPrompt~\cite{mao2025parse} further explores structure-guided prompt compression by introducing parse trees and formulating compression as a tree construction and pruning problem.
These studies suggest that structural information is useful for prompt compression, but multi-granularity semantic relations and redundancy across distant prompt segments remain insufficiently explored.

\subsection{Hierarchical Text Modeling}
Hierarchical modeling captures multi-level semantic organization in long documents.
Prior work~\cite{ruan2022histruct+,zangari2024hierarchical,ahmad2025hierarchical,zhao2025can} models text at multiple granularities (tokens, sentences, paragraphs) to capture both local semantic organization and document-level dependencies.
Such hierarchical representations improve the modeling of long-range interactions while preserving fine-grained contextual information that is often lost in flat sequential encoding schemes.
Recent approaches further organize these units through multiplex or multi-relational structures~\cite{sha2024hierarchical,yu2022multiplex,behrouz2022cs,shen2024beyond}, co-embedding different types of relations as interdependent layers.
Despite these advances, most hierarchical text modeling methods are designed for representation learning, document understanding, or downstream prediction, rather than for compression under explicit token budgets.
As a result, they usually preserve or encode multi-level textual information instead of deciding which units should be removed, merged, or retained during prompt compression.
This distinction is important because prompt compression requires not only modeling hierarchical dependencies, but also estimating redundancy and task relevance across levels.
Therefore, hierarchical modeling provides useful structural motivation, but efficient budget-constrained compression remains a separate problem.

\subsection{Graph-based Methods and Traversal Strategies}
Graph-based representations provide explicit structural abstraction for text~\cite{ruan2022histruct+,zhao2025leveraging}, modeling words or sentences as nodes connected by syntactic, semantic, or similarity-based edges.
Such formulations support multi-level interactions and long-range dependencies~\cite{yu2022multiplex}, and have been applied to text classification~\cite{onan2023hierarchical}, summarization~\cite{ruan2022histruct+}, and question answering~\cite{sui2025fidelis,xu2025harnessing}.
Recent graph-based LLM systems further extend this idea to retrieval, reasoning, and context organization. 
Graph-RAG methods organize external corpora into entity or document graphs for multi-hop retrieval and question answering~\cite{edge2024local}. 
Zarrinkia et al.~\cite{zarrinkia2026reasoning} further identify a reasoning bottleneck in Graph-RAG and propose structured prompting with graph-walk context compression for multi-hop QA. 
However, these methods mainly operate on retrieved knowledge-graph contexts rather than directly pruning redundancy inside the input prompt.
For prompt compression, Prompt-SAW~\cite{ali2024prompt} constructs relation-aware graphs from entity-relation triplets and selects subgraphs via similarity-based scoring. 
AttnComp~\cite{zhao2025leveraging} uses query-to-context attention and graph-based token clustering to form semantic units for prompt compression. 
These methods show the value of structural dependencies, but they mainly focus on entity-relation graphs, token clustering, or deterministic scoring, without jointly modeling fine-grained local dependencies and sentence-level semantic transitions in a unified graph structure.
A key challenge in graph-based importance estimation is how to traverse the graph effectively. 
Standard random-walk and diffusion-based methods often rely on local connectivity patterns and homogeneous graph assumptions, which may limit their ability to capture both local semantic coherence and global structural diversity in hierarchical text graphs. 
This limitation suggests that prompt compression requires traversal strategies that can jointly preserve local semantic coherence and global evidence coverage under strict token budgets.

\subsection{Efficient Exploration in Multiplex Networks}
Efficient exploration strategies on complex networks have been widely studied in network science and information retrieval. Traditional random walk methods estimate node importance through local neighborhood transitions~\cite{brin1998anatomy,tong2006fast}, while PageRank-based approaches further model global importance propagation via iterative diffusion over graph structures~\cite{bianchini2005inside}.
These methods have been widely applied to ranking, recommendation, and graph representation learning.
However, both standard random walks and PageRank variants often rely on local connectivity patterns and homogeneous graph assumptions, which may lead to limited exploration capability on sparse or highly clustered graphs.
To improve exploration efficiency, stochastic traversal strategies with long-range transitions have been introduced.
Among them, Lévy walks employ heavy-tailed step-length distributions that combine frequent short-range exploration with occasional long-distance jumps~\cite{guo2016levy,roy2020learning}.
Recent studies further extend Lévy walk-based exploration to multiplex networks for modeling heterogeneous interactions across multiple graph layers~\cite{guo2016levy,roy2020learning,boccaletti2014structure}.
Compared with conventional random walk and diffusion strategies, Lévy walks have shown stronger capability in balancing local exploration and global coverage, especially in complex networks with sparse long-range connections.
Nevertheless, existing approaches mainly focus on generic diffusion or navigation problems and are rarely designed for long-context text compression.
Moreover, many methods assume homogeneous node correspondence or fixed inter-layer transitions, which may limit their ability to model hierarchical semantic structures and adaptive information selection in textual graphs~\cite{wang2020heterogeneous,jing2021multiplex}.
These limitations suggest that long-context prompt compression requires traversal strategies that can operate over cross-granularity textual graphs while balancing local semantic coherence and global information coverage.
This setting is also different from GNN-based text modeling: GNNs usually learn node representations through message passing and task-specific supervision, whereas graph-based prompt pruning uses graph structure as an inference-time mechanism for selecting content under a token budget.

\begin{figure*}[!t]
    \centering
    \includegraphics[width=1\textwidth]{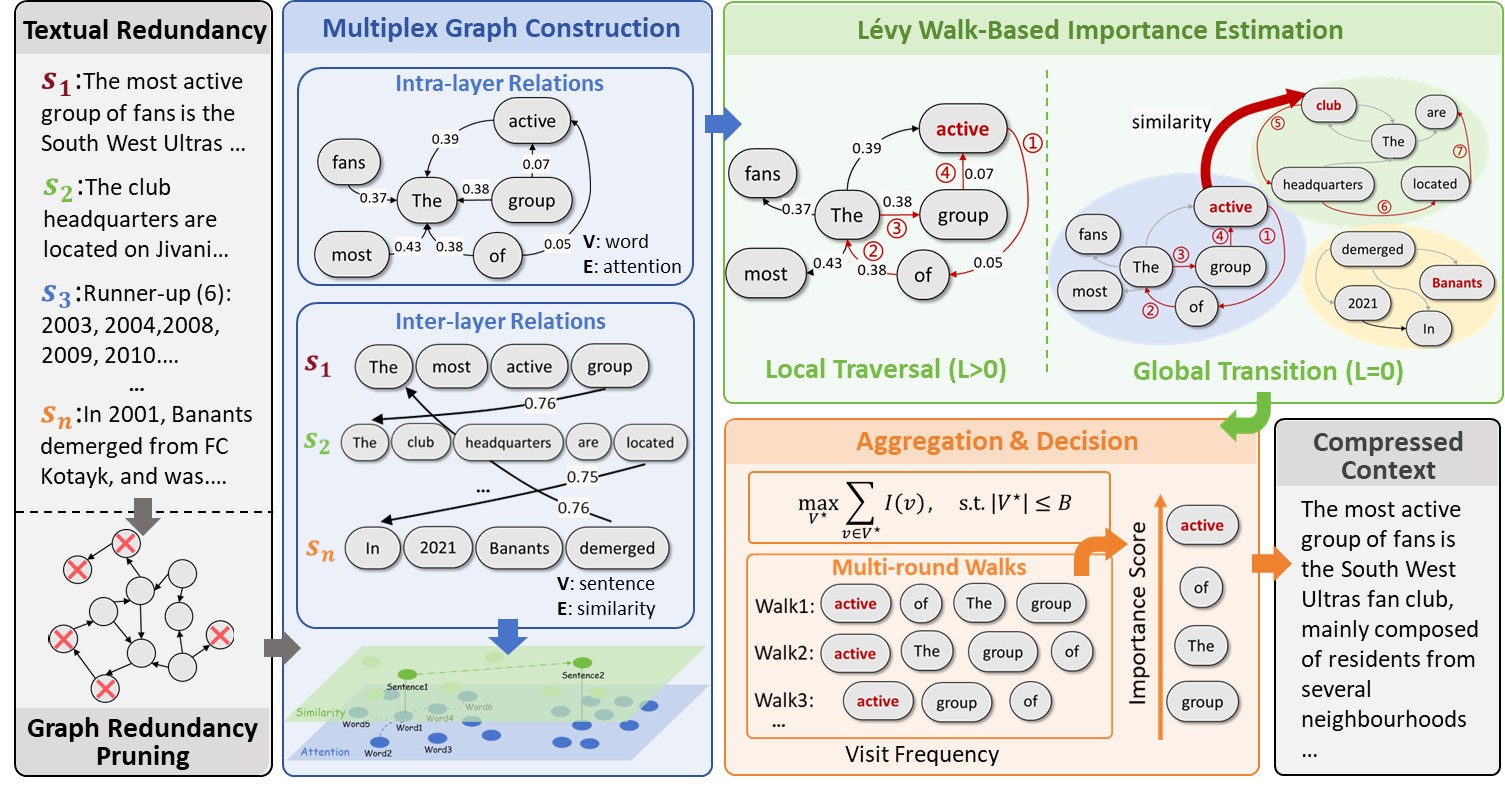}
    \caption{Framework of the proposed approach.
Textual redundancy is converted into graph redundancy pruning on a multiplex graph, where multi-round Lévy walks estimate node importance for text compression.
}
    \label{fig:overview}
\end{figure*}

\section{Multiplex Graph Construction}

\label{sec:graph_construction}
We propose RAGP, a graph-based prompt compression method that models long prompts as a multiplex graph and estimates fine-grained importance via stochastic Lévy walks, as illustrated in Figure~\ref{fig:overview}.
We begin by describing the multiplex graph construction, which forms the basis for subsequent importance estimation and compression.

\noindent\textbf{Notations.}
Let the input prompt be a sequence of $K$ sentences $\mathcal{D} = \{ s_1, s_2, \dots, s_K \}$, where each sentence $s_k$ consists of semantic units (e.g., words or subwords), denoted as $V^{(0)}_{s_k} = \{ v_{k,1}, \dots, v_{k,n_k} \}$. Prompt compression aims to select a subset of semantic units under a given budget while preserving task-relevant information. To model dependencies at multiple granularities, we represent the prompt as a multiplex graph $\mathcal{G} = \{ G^{(0)}, G^{(1)} \}$, where the fine-grained layer $G^{(0)}$ contains semantic-unit nodes and the coarse-grained layer $G^{(1)}$ contains sentence nodes. All notations used in the paper are summarized in Appendix~\ref{app:notation}.

We construct a multiplex graph over the input prompt to capture semantic dependencies at multiple granularities.  
To ensure scalability on long contexts, we apply a lightweight pre-filtering step that removes sentences with low query relevance, bounding the graph size before analysis.  
This preserves content for fine-grained importance estimation.

Based on the retained content, we instantiate semantic units at the word level (each possibly composed of multiple subword tokens) and coarse-grained nodes at the sentence level.
The resulting multiplex graph is denoted $\mathcal{G} = { G^{(0)}, G^{(1)} }$, where the fine-grained layer $G^{(0)} = (V^{(0)}, E^{(0)})$ encodes local dependencies among semantic units, and the coarse-grained layer $G^{(1)} = (V^{(1)}, E^{(1)})$ captures global semantic relations across sentences.

\paragraph{Intra-layer Relations (Local Structure).}
The fine-grained layer encodes local semantic dependencies among words. Each node represents a semantic unit instantiated at the word level, which may consist of one or more subword tokens.
Edges in the fine-grained layer are instantiated based on attention patterns
produced by a pretrained language model. These edge weights represent the contextual dependency between semantic units (words/subwords). For two words \( v_i \) and \( v_j \) that belong to the same sentence \( s \), we introduce an edge
\( (v_i, v_j) \in E^{(0)}_s \).
The corresponding edge weight is defined by aggregating token-level attention
scores between all subword tokens of \( v_i \) and \( v_j \):
\begin{equation}
\begin{aligned}
w^{(0)}(v_i, v_j)
&= \frac{1}{|\mathrm{Tok}(v_i)| \, |\mathrm{Tok}(v_j)|} \\
&\quad \times
\sum_{t \in \mathrm{Tok}(v_i)}
\sum_{t' \in \mathrm{Tok}(v_j)}
\mathrm{Attn}(t, t').
\end{aligned}
\end{equation}
where $\mathrm{Tok}(v)$ denotes the set of subword tokens composing semantic unit $v$, and $\mathrm{Attn}(t, t')$ denotes the attention weight from token $t$ to $t'$, averaged across all attention heads of a pretrained language model.
Based on this aggregation, to mitigate noise induced by weak or spurious attention links, we apply a sparsification step that retains only the top $\delta\%$ of edges ranked by weight, i.e., \( (v_i, v_j) \in E^{(0)} \) if \( w^{(0)}(v_i, v_j) \) is among the top $\delta\%$.
This aggregation yields a dense word-level graph within each local context,
where attention-induced neighborhood overlap leads to substantial redundancy
among word nodes.
While attention provides a general and task-agnostic signal for modeling local dependencies, the graph construction can be flexibly augmented with task-specific cues when available.
For example, in code-related tasks, syntactic or structural relations derived from the abstract syntax tree (AST), such as data or control dependencies, can be readily incorporated to complement attention-based connectivity.

\paragraph{Inter-layer Relations (Global Structure).}
The coarse-grained layer \( G^{(1)} = (V^{(1)}, E^{(1)}) \) captures global semantic dependencies across sentences.
Each node \( s \in V^{(1)} \) corresponds to a sentence in the input text, serving as its graph-level representation.
For two sentence nodes $s, s' \in V^{(1)}$, we introduce a directed edge $(s, s') \in E^{(1)}$ if $s'$ is among the top similar sentences of $s$.
Sentence-level semantic similarity is computed using all-MiniLM-L6-v2 sentence embeddings with cosine similarity, and the edge weight is defined as
\begin{equation}
    w^{(1)}(s, s') = \mathrm{Sim}(s, s') = \cos(\mathbf{e}_s, \mathbf{e}_{s'}),
\end{equation}
where $\mathbf{e}_s$ denotes the sentence embedding of $s$.
To limit global connectivity, we retain only the top $30\%$ of inter-sentence edges based on semantic similarity, producing a sparse graph that facilitates long-range information flow while complementing local dependencies..

\paragraph{Multiplex Coupling.}
The two layers are coupled via the hierarchical mapping \( \pi : V^{(0)} \rightarrow V^{(1)} \), assigning each fine-grained node to its corresponding sentence.
This structure enables traversal across dense local neighborhoods and sparse global links, naturally motivating Lévy walks for importance estimation and providing a heterogeneous foundation for graph-based node selection.

\section{Lévy Walk-based Importance Estimation}

\paragraph{Why Lévy Walks on Multiplex Graphs?}

We provide intuition for why Lévy walks are well-suited for importance estimation on multiplex text graphs.
The key insight is that the performance advantage of Lévy walks depends on the \emph{heterogeneity} of the multiplex graph structure.

\begin{proposition}[Sentence Coverage]
\label{prop:coverage}
Let $\mathcal{G}$ be a multiplex graph with $K$ sentences,
average intra-sentence degree $d_{\text{local}}$, and average inter-sentence degree $d_{\text{global}}$.
Define the heterogeneity ratio $\eta = d_{\text{local}}/d_{\text{global}}$.
The expected number of steps to visit all sentences satisfies:
\begin{align}
T_{\text{RW}} &= \Theta(K \cdot \eta \cdot \ln K), \\
T_{\text{L\'{e}vy}} &= \Theta\left(K \cdot \frac{\mu-1}{\mu-2} \cdot \ln K\right) \quad \text{for } \mu > 2.
\end{align}
Thus, Lévy walk achieves faster coverage when $\eta > (\mu-1)/(\mu-2)$, as the $\ln K$ factors cancel in the speedup ratio.
\end{proposition}

\textbf{Interpretation.}
The speedup ratio $\eta(\mu-2)/(\mu-1)$ reveals the key factors affecting performance:
\begin{itemize}
\item \textbf{Document structure} ($\eta$): Long documents with diverse topics exhibit high $\eta$,
as sentences are internally coherent ($d_{\text{local}}$ high) but loosely connected globally ($d_{\text{global}}$ low).
This is precisely where Lévy walks excel.
\item \textbf{Lévy exponent} ($\mu$): Controls the exploration-exploitation balance.
Smaller $\mu$ increases exploration (more frequent global jumps),
while larger $\mu$ favors exploitation (longer local traversal).
For typical long-context tasks, $\mu \approx 2.5$ provides a good balance.
\end{itemize}
Conversely, for short documents with uniform structure ($\eta \approx 1$),
both methods perform comparably.
For $\mu = 2.5$, the threshold is $\eta > 3$; empirically, we observe $\eta \in [15, 50]$ on LongBench tasks, well above this threshold (see Appendix~\ref{app:proof} for analysis and document-type predictions).

\textbf{Proof sketch.}
A standard random walk on $G^{(0)}$ transitions between sentences with probability $\approx 1/\eta$,
leading to expected waiting time $\approx \eta$ per new sentence.
By coupon collector arguments, visiting all $K$ sentences requires $T_{\text{RW}} = \Theta(K \cdot \eta \cdot \ln K)$ steps.

In contrast, Lévy walks perform a global jump at the end of each segment,
with expected segment length $\mathbb{E}[L] = (\mu-1)/(\mu-2)$ for $\mu > 2$.
This yields $T_{\text{L\'{e}vy}} = \Theta(K \cdot \mathbb{E}[L] \cdot \ln K)$.
Since both formulas share the $\ln K$ factor, the speedup ratio simplifies to $\eta / \mathbb{E}[L] = \eta(\mu-2)/(\mu-1)$, independent of $K$.
The full proof is in Appendix~\ref{app:proof}.

From a foraging theory perspective~\cite{viswanathan1999optimizing},
Lévy flights are optimal for locating sparsely distributed targets in heterogeneous environments.
Our multiplex graph exhibits analogous structure:
important semantic units are \emph{sparse} at the global scale (across sentences)
but \emph{clustered} at the local scale (within each sentence).

\begin{algorithm}[!t]
\caption{L\'evy Walk-based Importance Estimation}
\label{alg:levy_walk}
\begin{algorithmic}[1]
\REQUIRE Multiplex graph $\mathcal{G} = \{G^{(0)}, G^{(1)}, \pi\}$;
number of walks $N$; walk length $T$; L\'evy exponent $\mu$
\ENSURE Importance scores $\{ I(v) \mid v \in V^{(0)} \}$
\STATE Initialize $c(v) \gets 0$ for all $v \in V^{(0)}$
\FOR{$i = 1$ to $N$}
    \STATE Sample $v \sim \mathcal{U}(V^{(0)})$, set $s \gets \pi(v)$
    \STATE Sample $L \gets \lfloor (1-\mathcal{U}(0,1))^{-\frac{1}{\mu-1}} \rfloor$
    \FOR{$t = 1$ to $T$}
        \STATE $c(v) \gets c(v) + 1$
        \IF{$L > 0$}
            \STATE Sample $v' \in V^{(0)}_s$ with prob. $\propto w^{(0)}(v,v')$, set $v \gets v'$
            \STATE $L \gets L - 1$
        \ELSE
            \STATE Sample $s' \in V^{(1)}$ with prob. $\propto w^{(1)}(s,s')$
            \STATE Sample $v' \in V^{(0)}_{s'}$ uniformly, set $(v,s) \gets (v',s')$
            \STATE Reset $L \gets \lfloor (1-\mathcal{U}(0,1))^{-\frac{1}{\mu-1}} \rfloor$
        \ENDIF
    \ENDFOR
\ENDFOR
\STATE Normalize: $I(v) = c(v)/\sum_{v' \in V^{(0)}} c(v')$
\STATE \textbf{return} $I(v)$
\end{algorithmic}
\end{algorithm}

\textbf{Why visit frequency reflects importance.}
Nodes that are (1) well-connected within their local context, or
(2) serve as semantic bridges between sentences,
will be visited more frequently during the walk.
Conversely, redundant nodes---whose information is captured by their neighbors---have
their visits diluted across alternatives.
Thus, visit frequency serves as a natural proxy for non-redundant informativeness.

We empirically validate this hypothesis through correlation analysis on the CNN/DailyMail test set ($N=500$).
As shown in Table~\ref{tab:visit_correlation}, the Spearman correlation between visit-frequency-based importance scores and human-annotated salience labels exhibits a mean value of 0.5046 and a median of 0.5367.
Furthermore, 97.8\% of samples show positive correlation, with 75.2\% reaching statistical significance ($p < 0.05$), indicating that traversal frequency is a reliable proxy for content importance estimation.

\paragraph{Method Overview.}
We estimate the importance of fine-grained nodes \(v \in V^{(0)}\) by defining a stochastic Lévy walk process on the multiplex graph \( \mathcal{G} \), where each fine-grained node is associated with a sentence \(s = \pi(v) \in V^{(1)}\). The goal is to assign an importance score to each node in \( V^{(0)} \) such that removing low-importance nodes yields a compact subgraph that preserves both local semantic coherence and global dependencies.
The overall procedure is summarized in Algorithm~\ref{alg:levy_walk}, with an illustrative example of the Lévy walk traversal shown in Figure~\ref{fig:example}.

\begin{table}[t]
\centering
\small
\begin{tabular}{lc}
\hline
\textbf{Metric} & \textbf{Value} \\
\hline
Spearman correlation (mean) & 0.5046 \\
Spearman correlation (median) & 0.5367 \\
Positive correlation ratio & 97.8\% \\
Statistically significant ($p < 0.05$) & 75.2\% \\
\hline
\end{tabular}
\caption{Correlation analysis between sentence coverage time (visit frequency) and human-annotated importance labels on the CNN/DailyMail test set ($N=500$).}
\label{tab:visit_correlation}
\end{table}

\vspace{1em} 
\begin{center}
    \begin{minipage}{0.48\textwidth}
        \centering
        \includegraphics[width=\linewidth]{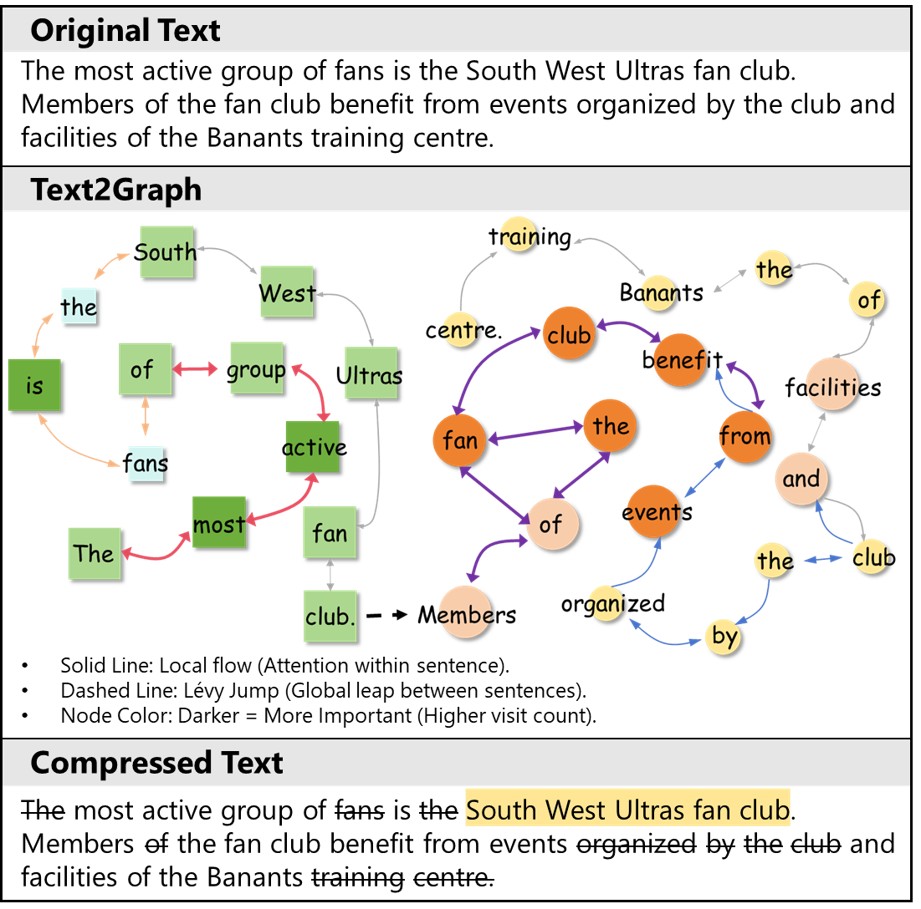}
        \captionof{figure}{Illustration of L\'{e}vy walk traversal. Solid edges show local exploration, dashed edges represent long-range jumps. Edge intensity reflects attention strength, and darker nodes are visited more frequently and retained after compression.}
        \label{fig:example}
    \end{minipage}
\end{center}
\vspace{1em}

\begin{table*}[t]
  \centering
  \small
  \caption{Performance comparison of different methods on the LongBench benchmark using GPT-3.5-Turbo. The symbol \textsuperscript{*} indicates results reproduced by our implementation. The best results are marked in \textbf{boldface}, and the sub-optimal ones are \underline{underlined}.}
  \vspace{-1mm}
    \begin{tabular}{l ccccccc cc}
    \toprule
    \multicolumn{1}{c}{\multirow{2}[4]{*}{\textbf{Methods}}} & \multicolumn{9}{c}{\textbf{LongBench}} \\
\cmidrule{2-10}          & SingleQA & MutilQA & Summ. & FewShot & Synth. & Code  & AVG   & Tokens & Ratio \\
    \midrule
    \multicolumn{10}{c}{\textbf{2,000 tokens constraint}} \\
    \midrule
    \multicolumn{10}{l}{\textit{Retrieval-based Methods}} \\
    BM25  & 30.1  & 29.4  & 21.2  & 19.5  & 12.4  & 29.1  & 23.6  & 1985  & 5× \\
    SBERT & 33.8  & 35.9  & 25.9  & 23.5  & 18.0  & 17.8  & 25.8  & 1947  & 5× \\
    OpenAI & 34.3  & 36.3  & 24.7  & 32.4  & 26.3  & 24.8  & 29.8  & 1991  & 5× \\
    \multicolumn{10}{l}{\textit{Compression-based Methods}} \\
    Selective-Context & 16.2  & 34.8  & 24.4  & 15.7  & 8.4   & 49.2  & 24.8  & 1925  & 5× \\
    LLMLingua & 22.4  & 32.1  & 24.5  & 61.2  & 10.4  & 56.8  & 34.6  & 1950  & 5× \\
    LLMLingua-2 & 29.8  & 33.1  & 25.3  & 66.4  & 21.3  & \textbf{58.9} & 39.1  & 1954  & 5× \\
    LongLLMLingua & 39.0  & 42.2  & \textbf{27.4} & \textbf{69.3} & 53.8  & 56.6  & \underline{48.0}  & 1809  & 6× \\
    CPC\textsuperscript{*}   & 36.9  & \textbf{44.8} & 23.3  & 58.8  & \underline{56.5}  & 45.6  & 44.3  & 1844  & 5× \\
    EFPC  & \underline{41.7}  & 42.4  & 25.8  & 67.3  & 27.0  & \underline{57.6}  & 43.6  & 1972  & 5× \\
    GPT-C & 36.3  & 38.5  & \underline{27.1}  & \underline{67.7}  & 46.3  & 53.1  & 44.8  & 1898  & 5× \\
    \rowcolor{gray!20}
    \textbf{RAGP (Ours)} & \textbf{41.9} & \underline{42.7}  & 22.4  & 62.1  & \textbf{62.0} & 57.1  & \textbf{48.1} & 1793  & \textbf{6×} \\
    \midrule
    \multicolumn{10}{c}{\textbf{3,000 tokens constraint}} \\
    \midrule
    \multicolumn{10}{l}{\textit{Retrieval-based Methods}} \\
    BM25  & 32.3  & 34.3  & 25.3  & 57.9  & 45.1  & 48.9  & 40.6  & 3417  & 3× \\
    SBERT & 35.3  & 37.4  & 26.7  & 63.4  & 51.0  & 34.5  & 41.4  & 3399  & 3× \\
    OpenAI & 34.5  & 38.6  & 26.8  & 63.4  & 49.6  & 37.6  & 41.7  & 3421  & 3× \\
    \multicolumn{10}{l}{\textit{Compression-based Methods}} \\
    Selective-Context & 23.3  & 39.2  & 25.0  & 23.8  & 27.5  & 53.1  & 32.0  & 3328  & 3× \\
    LLMLingua & 31.8  & 37.5  & 26.2  & 67.2  & 8.3   & 53.2  & 37.4  & 3421  & 3× \\
    LLMLingua-2 & 35.5  & 38.7  & 26.3  & 69.6  & 21.4  & 62.8  & 42.2  & 3392  & 3× \\
    LongLLMLingua & 40.7  & \underline{46.2} & \textbf{27.2} & \textbf{70.6} & 53.0  & 55.2  & \underline{48.8}  & 3283  & 3× \\
    CPC\textsuperscript{*}   & 41.2  & 46.1  & 23.7  & 61.7  & \underline{55.8}  & 50.9  & 46.5  & 3327  & 3× \\
    GPT-C & 39.9  & 40.6  & \underline{27.0}  & 68.4  & 48.1  &\underline{58.4}  & 47.1  & 3187  & 3× \\
    \rowcolor{gray!20}
    \textbf{RAGP (Ours)} & \textbf{43.6} & 44.7  & 22.7  & 64.2  & \textbf{62.5} & \underline{58.4}  & \textbf{49.3} & 2691  & \textbf{4×} \\
    \bottomrule
    \end{tabular}%
  \label{tab:main1}%
\end{table*}%

\paragraph{Stochastic Exploration on Multiplex Graphs.}
A Lévy walk is characterized by heavy-tailed step lengths, resulting in frequent short-range transitions interspersed with occasional long-range jumps~\cite{viswanathan1999optimizing,zaburdaev2015levy}.
This property naturally matches the heterogeneous structure of the multiplex graph, where intra-context subgraphs are dense and inter-context connections are sparse.
Formally, at step \( t \), the walker state is defined as $(v_t, s_t, L_t)$,
where \( v_t \in V^{(0)} \) is the current fine-grained node, \( s_t = \pi(v_t) \in V^{(1)} \) is its associated sentence node, and \( L_t \in \mathbb{N} \) denotes the remaining length of the current Lévy segment.
Segment lengths are drawn from a power-law distribution

\begin{equation}
   P(L = l) \propto l^{-\mu}, \quad \mu > 2,
\end{equation}
where larger \(\mu\) biases the walk towards shorter 
steps, while smaller \(\mu\) (closer to 2) allows more long-range jumps.

\paragraph{Transition Dynamics.}
We define a stochastic traversal process over the multiplex graph $\mathcal{G} = \{ G^{(0)}, G^{(1)} \}$.
At each step $t$, the walker maintains a fine-grained node $v_t \in V^{(0)}$ and its associated sentence node $s_t = \pi(v_t) \in V^{(1)}$.

At the beginning of each Lévy segment, we sample a random variable $r \sim \mathcal{U}(0,1)$ and derive a power-law-distributed segment length
\begin{equation}
L = \left\lfloor (1 - r)^{-\frac{1}{\mu - 1}} \right\rfloor,
\end{equation}
where $\mu > 2$ controls the tail heaviness of the distribution (ensuring finite expected segment length).
The variable $L$ specifies the number of consecutive local transitions to be performed within the current sentence $s_t$.

\textbf{Local traversal.}
While $L > 0$, the walker remains within the fine-grained subgraph $G^{(0)}_{s_t}$ induced by sentence node $s_t$.
At each step, a neighboring fine-grained node $v_{t+1} \in V^{(0)}_{s_t}$ is sampled according to the row-normalized edge weights:
\begin{equation}
\mathbb{P}(v_{t+1} = v' \mid v_t, s_t)
= \frac{w^{(0)}(v_t, v')}{\sum_{v'' \in V^{(0)}_{s_t}} w^{(0)}(v_t, v'')},
\end{equation}
where $w^{(0)}(v_t, v')$ is the weight of the edge $(v_t, v')$ in the fine-grained layer.
The segment length is then decremented as $L \leftarrow L - 1$.

\textbf{Global transition.}
When the Lévy segment is exhausted ($L = 0$), the walker performs a coarse-grained transition by sampling a sentence-level node $s_{t+1} \in V^{(1)}$ from the coarse-grained graph $G^{(1)}$:
\begin{equation}
\mathbb{P}(s_{t+1} = s')
\propto w^{(1)}(s_t, s'),
\end{equation}
where $w^{(1)}(s_t, s')$ denotes the edge weight in the coarse-grained layer.
A fine-grained node $v_{t+1} \in V^{(0)}_{s_{t+1}}$ is then sampled from the selected sentence, and a new Lévy segment length $L$ is resampled.

This transition mechanism induces frequent short-range exploration within sentences and occasional long-range jumps across sentences, thereby enabling efficient multi-scale information propagation without introducing an explicit local–global balancing parameter.

\begin{table*}[htbp]
  \centering
  \small
  \setlength{\tabcolsep}{3pt}
  \caption{Comparison with PartPrompt on BBCNews and Truncated arXiv under different token constraints. All results are obtained using Mixtral-8x7B-Instruct-v0.1 as the downstream generation model. The best results are marked in \textbf{boldface}, and the sub-optimal ones are \underline{underlined}.}
    \begin{tabular}{ccccccc|cccccc}
    \toprule
    \multirow{2}[2]{*}{Methods} & \multicolumn{6}{c|}{\textbf{BBCNews}}         & \multicolumn{6}{c}{\textbf{truncate arXiv}} \\
          & BLEU  & ROUGE1 & ROUGE2 & ROUGEL & Tokens & $1/\tau$   & BLEU  & ROUGE1 & ROUGE2 & ROUGEL & Tokens & $1/\tau$ \\
    \midrule
    LLM Generation & 11.31 & 42.32 & 18.60 & 39.22 & 261.30 & 3.29  & 8.77  & 39.18 & 16.00 & 36.22 & 643.80 & 4.65 \\
    \midrule
    \multicolumn{13}{c}{20\% tokens constraint} \\
    \midrule
    Selective-Context & 3.58  & 31.63 & 9.22  & 28.73 & 188.60 & 4.56  & 5.28  & 34.86 & 12.25 & 31.48 & 656.60 & 4.56 \\
    LLMLingua & 9.56  & 35.09 & 15.28 & 32.89 & 206.00 & 4.18  & 7.35  & 36.10 & 14.15 & 33.18 & 576.70 & 5.19 \\
    LongLLMLingua & 9.22  & 33.49 & 14.60 & 31.42 & 161.00 & 5.34  & 11.22 & 39.49 & 18.47 & \underline{36.77} & 599.60 & 5.00 \\
    LLMLingua2 & 5.65  & 36.22 & 11.95 & 32.82 & 172.60 & 4.99  & 5.75  & 35.96 & 12.39 & 32.63 & 615.10 & 4.87 \\
    PartPrompt & \underline{13.69} & \underline{43.54} & \underline{22.87} & \textbf{41.05} & 173.70 & 4.96  & \underline{11.73} & \underline{39.83} & \textbf{18.74} & \textbf{37.01} & 595.80 & 5.03 \\
    RAGP(Ours) & \textbf{22.53} & \textbf{52.74} & \textbf{26.70} & \underline{36.74} & 178.23 & 4.24  & \textbf{13.91} & \textbf{44.93} & \underline{17.30} & 27.99 & 644.51 & 4.87 \\
    \midrule
    \multicolumn{13}{c}{30\% tokens constrain} \\
    \midrule
    Selective-Context & 6.68  & 37.49 & 13.54 & 34.15 & 294.20 & 2.93  & 7.65  & 38.14 & 14.92 & 34.65 & 1048.10 & 2.86 \\
    LLMLingua & 10.22 & 38.55 & 16.66 & 35.88 & 278.20 & 3.09  & 10.39 & 39.23 & 17.47 & 36.26 & 867.90 & 3.45 \\
    LongLLMLingua & 10.11 & 36.97 & 15.96 & 34.48 & 257.80 & 3.34  & 14.62 & 42.61 & \underline{22.13} & \underline{39.86} & 894.90 & 3.35 \\
    LLMLingua2 & 9.45  & 41.86 & 16.71 & 38.39 & 258.10 & 3.34  & 7.86  & 38.86 & 14.91 & 35.63 & 898.40 & 3.33 \\
    PartPrompt & \underline{18.38} & \underline{48.57} & \underline{28.39} & \textbf{46.21} & 257.70 & 3.34  & \underline{15.35} & \underline{43.12} & \textbf{22.50} & \textbf{40.26} & 839.40 & 3.57 \\
    RAGP(Ours) & \textbf{24.44} & \textbf{54.68} & \textbf{28.51} & \underline{38.41} & 251.51 & 3.01  & \textbf{16.62} & \textbf{47.39} & 20.31 & 30.60 & 971.41 & 3.29 \\
    \midrule
    \multicolumn{13}{c}{50\% tokens constrain} \\
    \midrule
    Selective-Context & 14.26 & 46.74 & 22.63 & 43.41 & 481.00 & 1.79  & 13.63 & 44.56 & 21.78 & 41.32 & 1763.70 & 1.70 \\
    LLMLingua & 15.71 & 45.98 & 23.21 & 43.17 & 435.80 & 1.97  & 17.26 & 45.85 & 25.00 & 43.08 & 1445.90 & 2.07 \\
    LongLLMLingua & 14.65 & 44.31 & 21.27 & 41.33 & 432.70 & 1.99  & 18.79 & 46.71 & \underline{26.48} & \underline{43.82} & 1497.40 & 2.00 \\
    LLMLingua2 & 17.65 & 50.12 & 26.06 & \underline{46.89} & 430.90 & 2.00  & 14.07 & 44.95 & 21.92 & 42.03 & 1500.90 & 2.00 \\
    PartPrompt & \textbf{28.28} & \underline{56.48} & \textbf{38.30} & \textbf{54.29} & 424.90 & 2.03  & \underline{19.50} & \underline{46.78} & \textbf{26.87} & \textbf{44.26} & 1405.80 & 2.13 \\
    RAGP(Ours) & \underline{28.24} & \textbf{57.68} & \underline{32.16} & 41.50 & 402.84 & 1.93  & \textbf{19.57} & \textbf{49.83} & 23.46 & 33.39 & 1621.64 & 2.00 \\
    \bottomrule
    \end{tabular}%
  \label{tab:PartPrompt}%
\end{table*}%

\begin{table*}[t]
  \centering
  \caption{Performance comparison of different methods on four representative LongBench sub-tasks. The best results are marked in \textbf{boldface}, and the sub-optimal ones are \underline{underlined}.}
  \resizebox{\textwidth}{!}{
  \scriptsize
    \begin{tabular}{lcccccccccccccc}
    \toprule
    \multicolumn{1}{c}{\multirow{2}[4]{*}{\textbf{Model}}} & \multicolumn{4}{c}{\textbf{Single-Doc QA}} & \multicolumn{4}{c}{\textbf{Summ.}}     & \multicolumn{4}{c}{\textbf{Fewshot}}   & \multicolumn{2}{c}{\textbf{Code}} \\
\cmidrule{2-15}    \multicolumn{1}{c}{} & QP    & NQA   & QA Zh & QA En & QSUM  & GovRep & News  & VcSum & TREC  & TriQA & Sam   & Lsht  & RB    & LCC \\
    \midrule
    \multicolumn{15}{c}{\textbf{Full-Context (No Compression)}} \\
    \midrule
    \rowcolor{gray!20}
    GPT-3.5-Turbo & 47.1  & 25.1  & 56.5  & 55.2  & 20.6  & 25.7  & 25.1  & 17.3  & 63.0  & 91.7  & 39.4  & 53.5  & 58.5  & 62.5 \\
    LLaMA-3.1-8B-Instruct & 44.6  & 26.3  & 62.2  & \textbf{55.0} & \textbf{23.3} & \textbf{32.4} & \underline{24.2}  & \underline{16.2}  & 19.3  & 89.1  & 7.6   & 0.0   & 42.8  & 46.4 \\
    Qwen2.5-7B-Instruct-1M & \underline{45.3}  & 25.6  & 61.0  & 49.8  & 23.0  & \underline{30.0}  & 18.6  & 12.1  & 59.4  & 86.9  & 36.5  & 42.0  & 29.8  & 21.7 \\
    Qwen3-8B & 44.7  & 26.1  & \textbf{63.2} & 52.9  & 19.6  & 26.9  & 23.9  & \underline{16.2}  & \underline{70.5}  & 88.0  & 36.2  & \underline{47.4}  & 40.9  & 44.9 \\
    GLM-4-9B-Chat-1M & 43.8  & \underline{26.7}  & \underline{63.0}  & 53.6  & 22.8  & 27.6  & 21.0  & 12.2  & 61.5  & \underline{90.1}  & \underline{39.2}  & 28.7  & 55.6  & \textbf{59.5} \\
    \midrule
    \multicolumn{15}{c}{\textbf{Visual-based Compression (3-4x)}} \\
    \midrule
    Glyph & 40.6  & \textbf{28.5} & 37.2  & 45.9  & 19.8  & 25.5  & 21.5  & 12.4  & \textbf{82.6} & 88.5  & 32.5  & 44.4  & \textbf{60.8} & 48.9 \\
    \midrule
    \multicolumn{15}{c}{\textbf{Text-based Compression (3-4x)}} \\
    \midrule
    \rowcolor{gray!20}
    \textbf{RAGP(Ours)} & \textbf{46.7} & 23.9  & 48.9  & \underline{54.9}  & \underline{23.1}  & 25.4  & \textbf{25.1} & \textbf{17.2} & 70.0  & \textbf{91.8} & \textbf{41.5} & \textbf{53.5} & \underline{58.0}  & \underline{58.8} \\
    \bottomrule
    \end{tabular}%
  \label{tab:main2}%
  }
\end{table*}%

\paragraph{Importance Aggregation and Graph Pruning.}
During traversal, each visit to a fine-grained node \( v \in V^{(0)} \)
increments a counter \( c(v) \).
After executing multiple independent walks, we estimate node importance as
\vspace{-1mm}
\begin{equation}
I(v) = \frac{c(v)}{\sum_{v' \in V^{(0)}} c(v')}.   
\end{equation}
Prompt compression is formulated as a graph pruning problem:
we select a subset \( V^\star \subseteq V^{(0)} \) that maximizes
the total retained importance
\vspace{-1mm}
\begin{equation}
    \max_{V^\star} \sum_{v \in V^\star} I(v),\quad\text{s.t. } |V^\star| \le B,
\end{equation}
where \( B \) is a budget determined by the target compression ratio.
The compressed prompt is obtained by projecting nodes in \(V^\star\)
back to their original textual order. Nodes with higher visit counts are prioritized, ensuring that essential semantic information is retained.

\section{Experiments}

\subsection{Experimental Setup}
\textbf{Datasets and Tasks.} We evaluate RAGP on \textbf{LongBench}~\cite{bai2023longbench}, a bilingual benchmark for long-context understanding across tasks. The benchmark includes single-document QA (\textbf{SingleDoc}), such as Qasper (QP) and NarrativeQA (NQA); multi-document QA (\textbf{MultiDoc}); summarization tasks (\textbf{Summ.}), such as QMSum (QSUM) and GovReport (GovRep); few-shot learning (\textbf{FewShot}), e.g., TREC and TriviaQA (TriQA); synthetic reasoning tasks (\textbf{Synth.}); and code-related tasks (\textbf{Code}), including RopoBench-P (RB) and LCC. Evaluation metrics reported in this paper are computed on the LongBench test sets. 
For the additional comparison with PartPrompt~\cite{mao2025parse}, we also evaluate RAGP on BBCNews~\cite{bbc_2024} and Truncated arXiv~\cite{arxiv_2024}, following the dataset and task settings of PartPrompt and reporting BLEU and ROUGE scores under different token constraints.

\textbf{Baselines and Settings.} We evaluate RAGP against a comprehensive set of baselines, grouped into three categories. 
(i) \textbf{Retrieval-based methods} (e.g., BM25~\cite{robertson2009probabilistic}, SBERT~\cite{reimers2019sentence}, OpenAI embeddings~\cite{openai_embeddings}) select relevant sentences or paragraphs under compression constraints. 
(ii) \textbf{Text-based compression methods} (e.g., LongLLMLingua ranker~\cite{jiang2023longllmlingua}, Selective-Context~\cite{li2023unlocking}, LLMLingua~\cite{jiang2023llmlingua}, LLMLingua-2~\cite{pan2024llmlingua}, CPC~\cite{liskavets2025prompt}, EFPC~\cite{cao2025efpc}, GPT-C~\cite{liu2025gpt}, PartPrompt~\cite{mao2025parse}) typically follow coarse-to-fine, structure-guided, or model-assisted strategies to preserve task-relevant content under compression budgets.
(iii) \textbf{Large-model and multimodal baselines} (e.g., LLaMA-3.1-8B-Instruct~\cite{llama3_1_8b_instruct}, Mixtral-8x7B-Instruct-v0.1~\cite{mixtral_8x7b_instruct_v0_1}, Qwen2.5-7B-Instruct~\cite{qwen2_5_7b_instruct}, Mistral-7B-Instruct-v0.3~\cite{mistral_7b_instruct_v0_3}, Qwen2.5-7B-Instruct-1M~\cite{qwen2_5_7b_instruct_1m}, Qwen3-8B~\cite{qwen3_8b}, GLM-4-9B-Chat-1M~\cite{glm4_9b_chat_1m}, Glyph~\cite{cheng2025glyph}
) include strong full-context models and the vision-based text compression Glyph method.
RAGP experiments were conducted on NVIDIA A100 GPUs.
Unless otherwise specified in the corresponding experiment, all methods are evaluated under the same task split, token budget, prompt format, downstream model, and decoding parameters within each comparison.
For the stochastic Lévy-walk procedure, we fix the random seed to 2 in the main experiments.

\subsection{Results}
\label{sec:Results}

\noindent\textbf{Main Results (Table~\ref{tab:main1}).}
We evaluate RAGP on LongBench using GPT-3.5-Turbo under strict token constraints.
LongBench is a challenging setting for prompt compression because many samples contain very long contexts, and task-relevant evidence may be distributed across distant parts of the input. Under a 2,000- or 3,000-token budget, the compressor must remove substantial content while preserving enough evidence for downstream reasoning.
Our method achieves the highest AVG score of $49.3$ under a $3,000$-token budget, surpassing the strongest baseline LongLLMLingua ($48.8$) while using fewer retained tokens and achieving a higher compression ratio ($4\times$ vs $3\times$).
RAGP performs particularly well on structure-dependent tasks: in Single-Doc QA, it scores $43.6$ ($+2.9$ over LongLLMLingua), and in Code tasks, it reaches $58.4$ ($+3.2$).
Even under a tighter $2,000$-token constraint, it maintains a competitive AVG score of $48.1$.
Therefore, the improvement should be interpreted together with token reduction: RAGP preserves comparable or better task performance while retaining fewer tokens under strict compression budgets.
These results demonstrate that our structured compression effectively reduces redundancy while preserving essential semantic structure, mitigating the ``lost-in-the-middle''~\cite{liu2024lost} issue.
Overall, the results highlight the robustness of RAGP across different task categories and compression budgets.

\noindent\textbf{Cross-Dataset Comparison with Structure-Guided Compression Methods (Table~\ref{tab:PartPrompt}).}
To further evaluate the generalizability of structure-guided prompt compression beyond LongBench, we compare RAGP with PartPrompt, a representative parse-tree-guided compression method, on BBCNews and Truncated arXiv under $20$\%, $30$\%, and $50$\% token constraints.
All results are obtained using Mixtral-8x7B-Instruct-v0.1 as the downstream generation model.
RAGP achieves the best BLEU and ROUGE-1 scores in most settings, indicating that multiplex-graph pruning is effective at preserving globally salient content after compression.
On BBCNews, RAGP obtains the highest BLEU, ROUGE-1, and ROUGE-2 scores under the $20$\% and $30$\% constraints, improving BLEU over PartPrompt by $8.84$ and $6.06$ points, respectively.
On Truncated arXiv, RAGP consistently achieves the best BLEU and ROUGE-1 scores across all compression budgets, suggesting better preservation of salient information in long scientific documents.
PartPrompt achieves higher ROUGE-L and several ROUGE-2 scores under larger token budgets, which suggests that parse-tree-guided pruning is effective at maintaining local syntactic continuity.
Overall, these results show that different structure-guided compression strategies have complementary strengths: parse-tree-based pruning better preserves local phrase-level continuity, whereas multiplex-graph pruning better captures global semantic dependencies and salient content.

\noindent\textbf{Comparison with Full-Context LLMs (Table~\ref{tab:main2}).}
Table~\ref{tab:main2} compares RAGP, evaluated with GPT-3.5-Turbo, against full-context LLMs and the vision-based text compression method \emph{Glyph}.
RAGP consistently achieves first- or second-best performance overall across sub-datasets, often surpassing models using uncompressed prompts. Notably, under the same backbone, RAGP maintains performance on par with the full-context model and even slightly outperforms it in several cases, despite substantial compression.
On the QP dataset, RAGP achieves $46.7$, outperforming Qwen2.5-7B-Instruct-1M ($45.3$) and GLM-4-9B-Chat-1M ($43.8$).
In the few-shot TriQA task, it achieves $91.8$, surpassing the strongest full-context model by $1.7$ points.
Compared to \emph{Glyph}, our text-domain approach shows clear advantages: in QA Zh and QA En, RAGP outperforms by $11.7$ and $9.0$ points respectively; in the code-related LCC task, our method ($58.8$) leads \emph{Glyph} ($48.9$) by nearly $10$ points, demonstrating that retaining raw textual tokens is critical for precision-heavy reasoning.

\begin{table}[htbp]
  \centering
  \small
  \caption{Cross-downstream-LLM generalization on the MultiFieldQA-en subtask. We report QA-F1 scores using Qwen2.5-7B-Instruct, LLaMA-3.1-8B-Instruct, and Mistral-7B-Instruct-v0.3 as downstream generation models.}
    \begin{tabular}{c|ccc}
    \toprule
    Methods & Qwen2.5 & LLaMA3.1 & Mistral-v0.3 \\
    \midrule
    LLMLingua & 32.34 & 32.34 & 32.52 \\
    LLMLingua2 & 37.82 & 37.00 & 40.83 \\
    LongLLMLingua & 32.29 & 31.91 & 32.60 \\
    \textbf{RAGP(Ours)} & \textbf{42.30} & \textbf{43.42} & \textbf{47.40} \\
    \bottomrule
    \end{tabular}%
  \label{tab:diffent_model}%
\end{table}%

\begin{table}[t]
\centering
\caption{Component-level preprocessing latency of RAGP on MultiFieldQA-en under the 4,000-to-3,000 token setting.}
\label{tab:component_latency}
\begin{tabular}{lcc}
\toprule
Component & Avg. time & Proportion \\
\midrule
Graph construction & 11.86s & 45.60\% \\
Lévy walk pruning & 14.09s & 54.20\% \\
Total preprocessing & 25.98s & 100.00\% \\
\bottomrule
\end{tabular}
\end{table}

\vspace{-1.2mm}
\subsection{Generalization, Efficiency, and Scalability Analysis}

\noindent\textbf{Cross-downstream-LLM Generalization.}
We further evaluate the generality of RAGP across different downstream generators. Specifically, we conduct a cross-downstream-LLM evaluation on the MultiFieldQA-en subtask using Qwen2.5-7B-Instruct, LLaMA-3.1-8B-Instruct, and Mistral-7B-Instruct-v0.3.

As shown in Table~\ref{tab:diffent_model}, RAGP consistently outperforms all compared compression baselines across the three downstream LLMs.
Specifically, RAGP achieves QA-F1 scores of $42.30$, $43.42$, and $47.40$ with Qwen2.5-7B-Instruct, LLaMA-3.1-8B-Instruct, and Mistral-7B-Instruct-v0.3, respectively.
Compared with the strongest baseline in each setting, RAGP improves QA-F1 by $4.48$, $6.42$, and $6.57$ points, respectively.
These results indicate that the proposed graph-based compression strategy is not tied to a particular downstream model, but generalizes well across different instruction-tuned LLMs.

\noindent\textbf{Efficiency and Cost Trade-off.}
We next analyze efficiency from two complementary perspectives: the preprocessing overhead introduced by RAGP and the downstream inference savings obtained after compression.

Let $K$ denote the number of retained sentences after query-relevance filtering, $d$ the sentence embedding dimension, $N$ the number of semantic-unit nodes, $E_f$ and $E_s$ the retained fine-grained and sentence-level edges, $W$ the number of Lévy walks, and $T$ the walk length. The fine-grained graph construction is dominated by Transformer attention extraction, with $\mathcal{O}(\sum_i L_i^2)$ attention-related cost over retained sentences. The sentence-level semantic graph requires all-pair cosine similarity over $K$ retained sentence embeddings, resulting in $\mathcal{O}(K^2 d)$ time and $\mathcal{O}(K^2)$ temporary memory before sparsification. After retaining the top-$\delta$ weighted edges, sparse graph storage is reduced to $\mathcal{O}(E_f + E_s)$. The Lévy walk stage performs $W$ walks of length $T$, leading to $\mathcal{O}(WT)$ time complexity, and the overall memory usage is $\mathcal{O}(N + K + E_f + E_s)$.

\begin{table}[t]
\centering
\small
\caption{Comparison of original and compressed prompts on MultiFieldQA using LLaMA-3.1-8B-Instruct.}
\begin{tabular}{lccc}
\toprule
\textbf{Metric} & Original & Compressed & Gain \\ 
\midrule
Input Tokens (avg.)     & 6956  & 2830 & \textbf{--59.3\%} \\
\midrule
Inference Latency (s)   & 13.83 & 13.07 & \textbf{--5.5\%} \\

\midrule
QA F1 Score (\%)           & 7.11  & 9.35 & \textbf{+2.24} \\
\bottomrule
\end{tabular}
\label{tab:compression_latency}
\end{table}

\vspace{1em} 
\begin{center}
    \begin{minipage}{0.48\textwidth}
        \centering
        \includegraphics[width=0.9\linewidth]{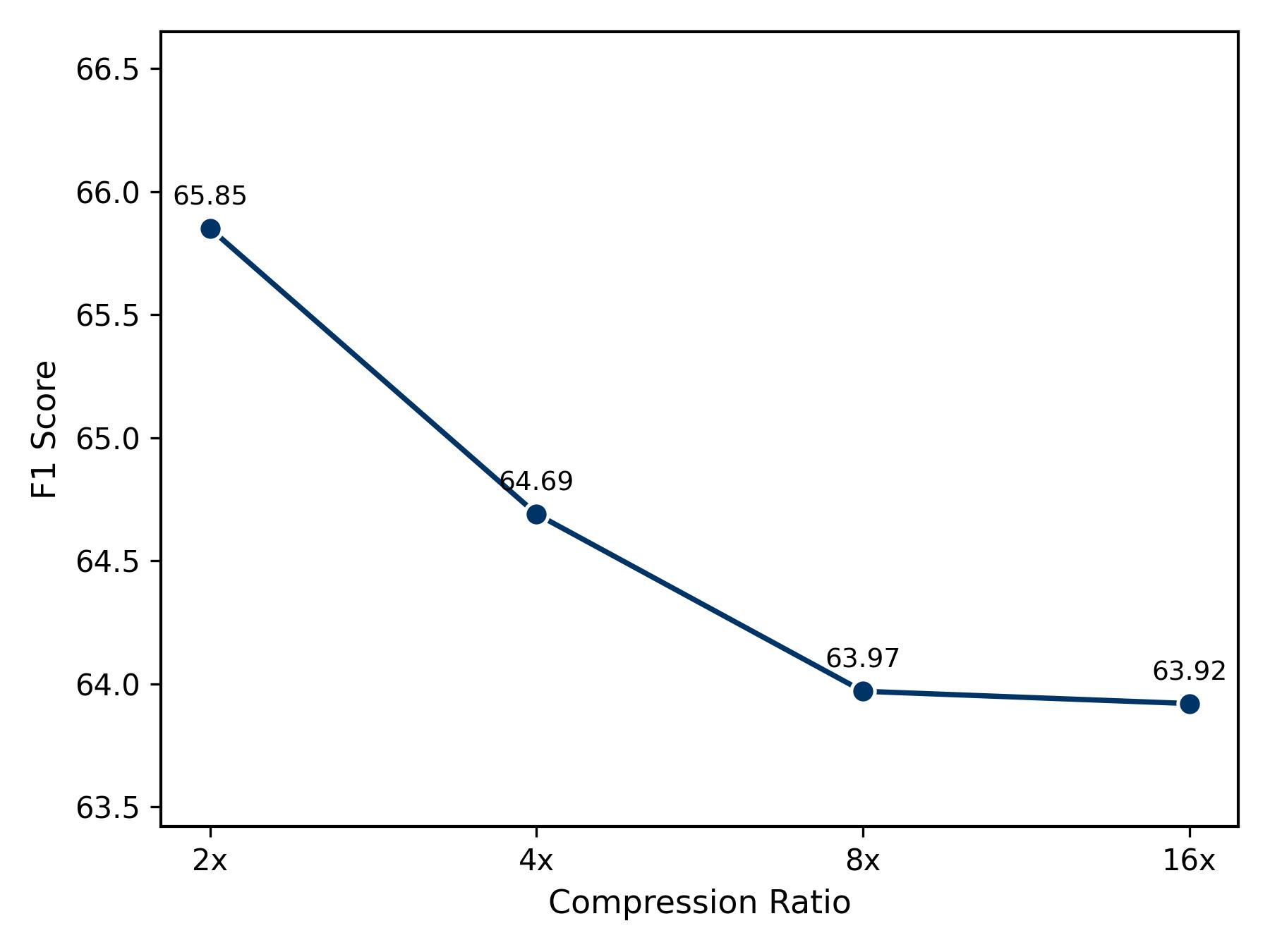}
        \captionof{figure}{F1 performance under different compression ratios. Performance decreases monotonically as the compression ratio increases, which is expected under more aggressive compression. This trend highlights the tradeoff between input compression and task performance.}
        \label{fig:ratio}
    \end{minipage}
\end{center}
\vspace{1em}

The empirical runtime distribution is consistent with the above complexity analysis.
As shown in Table~\ref{tab:component_latency}, graph construction takes 11.86 s on average, accounting for 45.60\% of preprocessing time, while Lévy-walk pruning takes 14.09 s, accounting for 54.20\%.
This confirms that the main preprocessing overhead comes from constructing and traversing the multiplex graph.

After preprocessing, we evaluate the downstream inference behavior of the compressed prompts.
Using LLaMA-3.1-8B-Instruct on MultiFieldQA-en under a controlled local setting, Table~\ref{tab:compression_latency} shows that RAGP reduces the average number of input tokens from $6,956$ to $2,830$ (a $59.3\%$ reduction), decreasing downstream inference latency from $13.83$s to $13.07$s.
Meanwhile, QA-F1 improves from $7.11$ to $9.35$.
These results indicate that removing redundant context can shorten the downstream input while preserving, and in this case improving, answer quality.

Beyond latency, token reduction can also reduce monetary cost in API-based deployment.
To isolate this effect, we conduct a accuracy--cost case study on MultiFieldQA-en using the same downstream API model, GPT-5.2, for both the original and compressed prompts, and compute cost under a unified API pricing scheme (Table~\ref{tab:final_cost}).
This experiment is intended to evaluate the cost impact of token reduction under a fixed API setting, rather than to compare with the GPT-3.5-Turbo main results or the LLaMA-3.1 local efficiency study.
As shown in Table~\ref{tab:final_cost}, RAGP maintains comparable F1 performance to the full-context prompt while reducing input tokens by $19.6\%$, leading to an $18.3\%$ reduction in inference cost.
These results suggest that prompt compression can reduce deployment cost in API-based settings, especially for large-scale or budget-constrained applications.

We further examine whether stronger compression harms task performance.
As shown in Figure~\ref{fig:ratio}, when the compression ratio increases from $2\times$ to $16\times$, F1 gradually decreases from 65.85 to 63.92.
The decline is expected under more aggressive compression, but the performance remains relatively stable at higher compression ratios, indicating that RAGP preserves core task-relevant information while removing redundancy.

Finally, Table~\ref{tab:end2end_efficiency} provides an end-to-end efficiency case study by comparing RAGP with lightweight compressors and the training-free Selective-Context baseline under the same LLaMA-3.1-8B-Instruct setting.
The results show a clear trade-off between preprocessing cost and compressed-prompt quality.
LLMLingua, LLMLingua-2, LongLLMLingua, and Selective-Context require lower preprocessing time, while RAGP achieves the best QA-F1 score (21.88) and the lowest downstream inference latency (1.89s), with fewer input tokens than LLMLingua, LongLLMLingua, and Selective-Context.
This suggests that RAGP's graph-based preprocessing improves the informativeness of the compressed prompt rather than merely reducing token length.
Therefore, RAGP is more suitable for offline preprocessing, reusable compressed contexts, and batch RAG scenarios where preprocessing cost can be amortized, while lightweight compressors may be preferable for latency-critical one-shot compression.

\begin{table}[htbp]
  \centering
  \small
  \caption{Accuracy vs. Cost. Pricing: Input \$1.75, Output \$14.00 per 1M tokens.}
    \begin{tabular}{lcccc}
    \toprule
    \textbf{Method} & \textbf{F1 (\%)} & \textbf{Tokens (In/Out) (k)} & \textbf{Cost} $\downarrow$ \\
    \midrule
    Original & \textbf{56.81} & 543.8 / 4.38 & \$1.013 \\
    \textbf{Ours} & 56.57 & \textbf{437.3} / 4.50 & \textbf{\$0.828} \\
    \midrule
    \textit{$\Delta$ (abs)} & \textit{-0.24} & \textit{-106.5 / +0.12} & \textit{-\$0.185} \\
    \textit{$\Delta$ (\%)} & \textit{(-0.4\%)} & \textit{(\textbf{-19.6\%} / +2.7\%)} & \textit{(\textbf{-18.3\%})} \\
    \bottomrule
    \end{tabular}
  \label{tab:final_cost}
\end{table}

\begin{table}[htbp]
  \centering
  \small
  \setlength{\tabcolsep}{4pt}
  \caption{End-to-end efficiency case study on MultiFieldQA-en using LLaMA-3.1-8B-Instruct as the downstream model. We report compression time, downstream inference time, average input tokens, and QA-F1 under the same evaluation setting.}
    \begin{tabular}{c|cccc}
    \toprule
    \textbf{Method} & \textbf{Input Tokens} & \textbf{Prep.} & \textbf{Infer.} & \textbf{F1} \\
    \midrule
    LLMLingua & 3125.98 & 0.66  & 1.96  & 16.08 \\
    LLMLingua2 & 2882.68 & 0.47  & 1.93  & 18.31 \\
    LongLLMLingua & 2991.47 & 0.47  & 1.91  & 15.09 \\
    Selective-Context & 3143.18 & 8.09  & 1.96  & 19.71 \\
    RAGP(Ours) & 2904.45 & 25.98 & 1.89  & 21.88 \\
    \midrule
    \end{tabular}%
  \label{tab:end2end_efficiency}%
\end{table}%

\noindent\textbf{Scalability and Stability.}
The complexity analysis above also clarifies the scalability boundary of RAGP.
The main bottleneck is sentence-level graph construction, whose all-pair similarity computation grows quadratically with the number of retained sentences.
Although query-relevance filtering reduces the number of sentences before graph construction and top-$\delta$ sparsification stores only the strongest edges, extremely long contexts containing hundreds of thousands or millions of tokens would require additional acceleration strategies.
Possible extensions include hierarchical chunking, approximate nearest-neighbor sentence graph construction, streaming graph updates, and multi-stage pruning.

Since RAGP uses Lévy walks for graph pruning, we also evaluate its stability under different random seeds.
As shown in Table~\ref{tab:multiseed_stability}, under a 3,000-token budget on MultiFieldQA-en, RAGP shows stable compression behavior across seeds 1, 2, and 3.
The average compression time is $14.77 \pm 0.11$s, and the average compressed length is $2763.50 \pm 1.66$ tokens.
These small standard deviations indicate that the random-walk procedure does not introduce substantial variance in the compression process under the same experimental setting.

\begin{table}[htbp]
  \centering
  \small
  \caption{Multi-seed stability analysis of RAGP on MultiFieldQA-en under a 3,000-token budget.}
  \begin{tabular}{c|cc}
    \toprule
    \textbf{Seed} & \textbf{Compression Time (s)} & \textbf{Compressed Tokens} \\
    \midrule
    1 & 14.83 & 2761.17 \\
    2 & 14.62 & 2764.56 \\
    3 & 14.88 & 2764.79 \\
    \midrule
    Mean $\pm$ Std. & 14.77 $\pm$ 0.11 & 2763.50 $\pm$ 1.66 \\
    \bottomrule
  \end{tabular}
  \label{tab:multiseed_stability}
\end{table}

\subsection{Ablation Study}
\label{sec:Ablation Study}

We conduct ablation studies on the Single-Doc QA task of LongBench using GPT-3.5-Turbo to assess the impact of key design choices, with results summarized in Table~\ref{tab:ablation_full}. Specifically, we examine three aspects: (i) the M2 graph refinement module and different graph traversal strategies, (ii) the sparsity threshold $\delta$ that controls edge pruning, and (iii) the Lévy exponent $\mu$ governing the balance between local and global transitions.

\noindent\textbf{Module Ablation and Graph Traversal Comparison.}
Before constructing the multiplex graph, we perform semantic pre-filtering to remove obviously irrelevant sentences based on similarity scores. This step, referred to as M1 in the ablation study, reduces graph size and improves efficiency.
We then compare three graph traversal strategies on top of M1: standard Random walk, PageRank, and our Lévy walk-based refinement (M2).
As shown in Table~\ref{tab:ablation_full}, Lévy walk (M1 + M2) consistently outperforms Random walk ($+2.26$) and PageRank ($+1.74$), demonstrating its effectiveness at selecting informative semantic units while reducing redundancy.
Adding the full graph refinement module M2 on top of M1 further improves the F1 score from $52.40$\% to $54.91$\%.
Since constructing the graph without semantic pre-filtering leads to substantially larger graph structures and significantly higher computational cost, we additionally evaluate standalone M2 under a smaller experimental setting on the Qasper dataset (GPT-3.5-Turbo, 3,000-token limit). As shown in Table~\ref{tab:qasper_ablation}, M1 alone achieves 44.08 F1, while M2 alone achieves 42.95 F1. Combining M1 and M2 further improves performance to 46.72 F1, indicating that semantic pre-filtering and graph-based refinement play complementary roles.
Overall, these results further validate the effectiveness of combining semantic pre-filtering with Lévy Walk-based graph refinement.

\begin{table}[t]
\centering
\caption{Ablation study and hyperparameter sensitivity of RAGP under a 3000-token budget.
\textbf{M1}: Semantic Pre-filtering;
\textbf{M2}: Multiplex Graph with Lévy Walk-based refinement.
Gain ($\Delta$) is computed relative to the M1-only baseline.}
\label{tab:ablation_full}
\small
\begin{tabular}{l c c c}
\toprule
\textbf{Setting} & \textbf{Configuration} & \textbf{F1 (\%)} & $\boldsymbol{\Delta}$  \\
\midrule
Baseline        & M1 only                & 52.40 & -- \\
Random          & M1 + Random            &52.65  & +0.25\\
PageRank        & M1 + PageRank            &53.17  & +0.77\\
Ours            & M1 + M2                & \textbf{54.91} & \textbf{+2.51} \\
\midrule
Graph sparsity  & $\delta = 0$ (Dense graph)            & 53.24 & +0.84 \\
                & $\delta = 30$ (Optimal)          & \textbf{54.91} & \textbf{+2.51} \\
                & $\delta = 50$ (Sparse graph)          & 53.45 & +1.05 \\

\midrule
Lévy parameter  & $\mu = 1.5$ (Global)            & 52.56 & +0.16 \\
                & $\mu = 2.5$  (Balanced)           & \textbf{54.91} & \textbf{+2.51} \\
                & $\mu = 3.5$ (Local)            & 53.67 & +1.27 \\
\bottomrule
\end{tabular}
\label{tab:ablation}%
\end{table}

\begin{table}[t]
\centering
\small
\caption{Additional ablation study on Qasper.}
\label{tab:qasper_ablation}
\begin{tabular}{lc}
\toprule
\textbf{Method} & \textbf{F1} \\
\midrule
M1 only (Semantic Pre-filtering) & 44.08 \\
M2 only (Lévy Walk) & 42.95 \\
M1 + M2 & \textbf{46.72} \\
\bottomrule
\end{tabular}
\end{table}

\noindent\textbf{Effect of Graph Sparsity Threshold ($\delta$).}
The graph sparsity threshold $\delta$ controls the trade-off between preserving graph connectivity and removing weak or noisy edges.
We fix the Lévy exponent at $\mu=2.5$ when analyzing $\delta$.
A dense graph ($\delta=0\%$) retains excessive weak edges, introducing noise that slightly reduces performance to $53.24\%$.
In contrast, overly aggressive pruning ($\delta=50\%$) removes some important semantic connections, lowering the F1 to $53.45\%$.
The chosen threshold $\delta=30\%$ strikes a balance, preserving the critical structural information for effective traversal and achieving the best score of $54.91\%$.
Thus, we use $\delta=30\%$ as an empirically validated default that reflects this connectivity--noise trade-off, rather than as a theoretically optimal constant.

\noindent\textbf{Sensitivity to Lévy Exponent ($\mu$).}
When studying the effect of $\mu$, we fix the graph sparsity at $\delta=30\%$.
We test values both within and outside the grounded regime ($\mu > 2$) to understand boundary behavior.
Smaller values (e.g., $\mu=1.5$, outside the $\mu > 2$ regime) induce long-range jumps with infinite expected segment length, leading to overly global exploration and yielding $52.56\%$. Larger values bias the walk toward short-range transitions; when $\mu=3.5$, the walk degenerates toward local exploration, slightly reducing performance to $53.67\%$. The intermediate setting $\mu=2.5$, which lies within the theoretically valid regime, achieves the best balance between local and global traversal and attains the highest F1 of $54.91\%$.
This empirically validates our theoretical analysis in Proposition~\ref{prop:coverage}: with $\mu = 2.5$, the expected segment length $\mathbb{E}[L] = (\mu-1)/(\mu-2) = 3$, providing an effective local-global transition ratio for the heterogeneous multiplex structure.

\subsection{Error Analysis}
\label{sec:Error Analysis}
We conduct an error analysis across different tasks in the LongBench benchmark to better understand the limitations of our approach.
We find that error patterns vary notably across task types due to their distinct structural characteristics.
For example, in code-related tasks, preserving structural elements such as function definitions, control flow, and variable dependencies is crucial; sentence-level importance estimation alone may remove syntactically essential components, leading to incorrect execution or reasoning.
In contrast, for narrative or multi-document question answering tasks, errors are more likely caused by the removal of low-frequency but query-relevant details, such as specific entities or fine-grained relations.
These observations suggest that incorporating task-aware or structure-aware constraints into the compression process could further improve robustness across diverse long-context tasks.

\section{Conclusion}

In this work, we present RAGP, a structured prompt compression framework for long-context LLM inputs. By modeling the input text as a multiplex graph and performing redundancy-aware pruning guided by Lévy walks, our approach effectively identifies and retains the most informative semantic units while removing redundant content. Extensive experiments on the LongBench benchmark demonstrate that RAGP consistently outperforms strong baselines across QA, summarization, few-shot, synthetic reasoning, and code tasks. Ablation studies further confirm the importance of local-global exploration, Lévy flight distribution, and attention-based graph sparsity.

\noindent\textbf{Limitations.} While RAGP performs strongly overall, it has several limitations.
First, multiplex graph construction introduces additional overhead.
Second, summarization performance is slightly below some strong baselines ($22.7$ vs. $27.2$).
Third, hyperparameters (e.g., $\mu$, $\delta$) may require tuning across domains.
Fourth, RAGP depends on the quality of the constructed multiplex graph. The fine-grained graph is derived from Transformer attention scores, while the sentence-level graph is constructed using MiniLM sentence embeddings. If these representations are noisy, domain-mismatched, or weak for a particular type of text, the graph may contain inaccurate edges or miss important dependencies, leading to suboptimal Lévy walks, inaccurate importance estimation, and degraded compression quality. Therefore, RAGP is model-independent in the sense that it does not require training or modifying the downstream LLM, but it is not representation-independent during graph construction. Changing the underlying attention or embedding model may alter the graph topology and the resulting compressed prompt.
Future work includes exploring adaptive, task-specific compression strategies, improving graph robustness through stronger domain-specific embeddings, edge calibration, or graph denoising, and extending the multiplex graph framework to handle multi-modal inputs.
Finally, the current implementation has not been evaluated on hundred-thousand- or million-token contexts, where additional scalability strategies would be required.

\appendix
\section{Notation}
\label{app:notation}
Table~\ref{tab:notation} summarizes the main symbols and notations used throughout the paper.

\section{Proof of Proposition~\ref{prop:coverage}}
\label{app:proof}

We provide a complete proof of Proposition~\ref{prop:coverage}, which characterizes the sentence coverage time for standard random walks versus Lévy walks on multiplex text graphs.

\begin{proposition}[Sentence Coverage, restated]
Let $\mathcal{G}$ be a multiplex graph with $K$ sentences,
average intra-sentence degree $d_{\text{local}}$, and average inter-sentence degree $d_{\text{global}}$.
Define the heterogeneity ratio $\eta = d_{\text{local}}/d_{\text{global}}$.
The expected number of steps to visit all $K$ sentences satisfies:
\begin{align}
T_{\text{RW}} &= \Theta(K \cdot \eta \cdot \ln K), \\
T_{\text{L\'{e}vy}} &= \Theta\left(K \cdot \frac{\mu-1}{\mu-2} \cdot \ln K\right) \quad \text{for } \mu > 2.
\end{align}
Thus, Lévy walk achieves faster coverage when $\eta > (\mu-1)/(\mu-2)$, as the $\ln K$ factors cancel in the speedup ratio.
\end{proposition}

\subsection{Model Setup}

Consider a multiplex graph $\mathcal{G} = \{G^{(0)}, G^{(1)}\}$ with the following structure:
\begin{itemize}
    \item \textbf{Fine-grained layer} $G^{(0)}$: Contains $K$ disjoint dense subgraphs, one per sentence. Each subgraph $G^{(0)}_s$ for sentence $s$ has $n_s$ nodes with average internal degree $d_{\text{local}}$.
    \item \textbf{Coarse-grained layer} $G^{(1)}$: A sparse graph over $K$ sentence nodes, with average degree $d_{\text{global}} \ll d_{\text{local}}$.
    \item \textbf{Coupling}: The mapping $\pi: V^{(0)} \to V^{(1)}$ associates each fine-grained node with its containing sentence.
\end{itemize}

We analyze two traversal strategies:
\begin{enumerate}
    \item \textbf{Standard Random Walk}: At each step, transitions among neighbors in an \emph{augmented} fine-grained graph, where inter-sentence edges are added with weights proportional to $w^{(1)}(s, s')$ from the coarse-grained layer. This models a baseline that has access to global connectivity information but does not exploit the multiplex structure explicitly. The inter-sentence transition probability is $p_{\text{escape}} \approx d_{\text{global}}/(d_{\text{local}} + d_{\text{global}}) \approx 1/\eta$ when $\eta \gg 1$.
    \item \textbf{Lévy Walk}: Alternates between local exploration within sentences and global jumps across sentences, with segment lengths drawn from a power-law distribution $P(L=l) \propto l^{-\mu}$. Unlike standard random walks, Lévy walks perform global jumps at \emph{deterministic} intervals (at the end of each segment), independent of the local graph structure.
\end{enumerate}

\begin{table*}[h]
\centering
\small
\caption{Notation used throughout the paper.}
\begin{tabular}{p{1.5cm} p{10cm}}
\toprule
Symbol & Description \\
\midrule
$\mathcal{D}$ & Input prompt (long text) \\
$K$ & Number of sentences \\
$s_k, s$ & The $k$-th sentence / sentence node in the coarse-grained graph \\
$v$ & Fine-grained node (semantic unit at word level) \\
$n_s, n_k$ & Number of fine-grained nodes in sentence $s$ (or $s_k$) \\
$\mathrm{Tok}(v)$ & Set of subword tokens composing semantic unit $v$ \\
$G^{(0)}$ & Fine-grained graph over semantic units $v$ \\
$G^{(1)}$ & Coarse-grained graph over sentence nodes $s$ \\
$V^{(0)}, V^{(1)}$ & Node sets of fine-grained and coarse-grained layers \\
$E^{(0)}, E^{(1)}$ & Edge sets of fine-grained and coarse-grained layers \\
$V^{(0)}_s$ & Fine-grained nodes within sentence $s$ \\
$G^{(0)}_s$ & Induced fine-grained subgraph within sentence $s$ \\
$\pi$ & Mapping from fine-grained nodes to sentence nodes: $\pi: V^{(0)} \to V^{(1)}$ \\
$w^{(0)}(v, v')$ & Edge weight in fine-grained layer (attention-based) \\
$w^{(1)}(s, s')$ & Edge weight in coarse-grained layer (semantic similarity) \\
$\mathbf{e}_s$ & Sentence embedding vector for sentence $s$ \\
$N$ & Number of independent random walks \\
$T$ & Length of each random walk (number of steps) \\
$L$ & Lévy segment length (steps before global jump) \\
$\mu$ & Lévy exponent controlling step-length distribution \\
$\delta$ & Threshold for edge sparsification \\
$B$ & Compression budget \\
$\eta$ & Heterogeneity ratio: $\eta = d_{\text{local}} / d_{\text{global}}$ \\
$I(v)$ & Importance score of node $v$ \\
$c(v)$ & Visit count of node $v$ during walks \\
\bottomrule
\end{tabular}
\label{tab:notation}
\end{table*}

\subsection{Analysis of Standard Random Walk}

\begin{lemma}[Inter-sentence transition probability]
\label{lem:rw_transition}
For a standard random walk on the fine-grained layer with weak inter-sentence connectivity,
the probability of transitioning from a node in sentence $s$ to a node in a different sentence $s' \neq s$ is:
\[
p_{\text{escape}} = \frac{d_{\text{global}}}{d_{\text{local}} + d_{\text{global}}} \approx \frac{1}{\eta + 1} \approx \frac{1}{\eta}
\]
when $\eta = d_{\text{local}}/d_{\text{global}} \gg 1$.
\end{lemma}

\begin{proof}
Consider a node $v$ in sentence $s$ with $d_{\text{local}}$ neighbors within $s$ and (on average) $d_{\text{global}}/n_s \cdot n_{s'}$ effective connections to other sentences through the coarse-grained layer.
Aggregating over all external sentences and normalizing:
\[
p_{\text{escape}} = \frac{\sum_{s' \neq s} w_{\text{inter}}(v, s')}{d_{\text{local}} + \sum_{s' \neq s} w_{\text{inter}}(v, s')} \approx \frac{d_{\text{global}}}{d_{\text{local}} + d_{\text{global}}} = \frac{1}{\eta + 1}.
\]
For $\eta \gg 1$, this simplifies to $p_{\text{escape}} \approx 1/\eta$.
\end{proof}

\begin{lemma}[Expected time to visit a new sentence]
\label{lem:rw_waiting}
Starting from a uniformly random node in a sentence $s$, the expected number of steps until the random walk first visits a node in a different sentence is $\Theta(\eta)$.
\end{lemma}

\begin{proof}
The time to escape sentence $s$ follows a geometric distribution with success probability $p_{\text{escape}} \approx 1/\eta$.
Therefore, the expected escape time is:
\[
\mathbb{E}[\tau_{\text{escape}}] = \frac{1}{p_{\text{escape}}} = \eta.
\]
\end{proof}

\begin{lemma}[Random walk sentence coverage time]
\label{lem:rw_coverage}
The expected number of steps for a standard random walk to visit all $K$ sentences is:
\[
T_{\text{RW}} = \Theta(K \cdot \eta \cdot \ln K).
\]
\end{lemma}

\begin{proof}
We apply the coupon collector argument. Let $T_i$ denote the time to visit the $i$-th new sentence after having visited $i-1$ distinct sentences.

When $i-1$ sentences have been visited, the probability that a random escape leads to a new (unvisited) sentence is approximately $(K - i + 1)/K$, assuming uniform distribution over sentences upon escape.

The expected time to visit the $i$-th new sentence consists of:
\begin{enumerate}
    \item Waiting time to escape the current sentence: $\mathbb{E}[\tau_{\text{escape}}] = \eta$ (by Lemma~\ref{lem:rw_waiting})
    \item Expected number of escapes needed to reach a new sentence: $K/(K-i+1)$ (coupon collector)
\end{enumerate}

Therefore:
\[
\mathbb{E}[T_i] = \eta \cdot \frac{K}{K - i + 1}.
\]

The total coverage time is:
\begin{align}
T_{\text{RW}} &= \sum_{i=1}^{K} \mathbb{E}[T_i] = \eta \sum_{i=1}^{K} \frac{K}{K - i + 1} \\
&= \eta \cdot K \sum_{j=1}^{K} \frac{1}{j} = \eta \cdot K \cdot H_K \\
&= \Theta(K \cdot \eta \cdot \ln K),
\end{align}
where $H_K = \sum_{j=1}^K 1/j = \Theta(\ln K)$ is the $K$-th harmonic number.
\end{proof}

\subsection{Analysis of Lévy Walk}

\begin{lemma}[Expected Lévy segment length]
\label{lem:levy_segment}
For a Lévy walk with exponent $\mu > 2$, the segment length $L$ drawn from $P(L = l) \propto l^{-\mu}$ for $l \geq 1$ has expected value:
\[
\mathbb{E}[L] = \frac{\zeta(\mu - 1)}{\zeta(\mu)} = \frac{\mu - 1}{\mu - 2} + O(1),
\]
where $\zeta(\cdot)$ is the Riemann zeta function.
\end{lemma}

\begin{proof}
The probability mass function is $P(L = l) = l^{-\mu} / \zeta(\mu)$ for $l \geq 1$.
The expected value is:
\[
\mathbb{E}[L] = \sum_{l=1}^{\infty} l \cdot \frac{l^{-\mu}}{\zeta(\mu)} = \frac{1}{\zeta(\mu)} \sum_{l=1}^{\infty} l^{1-\mu} = \frac{\zeta(\mu - 1)}{\zeta(\mu)}.
\]

For $\mu > 2$, using the asymptotic expansion $\zeta(s) \approx 1 + 2^{-s} + O(3^{-s})$ for large $s$:
\[
\frac{\zeta(\mu - 1)}{\zeta(\mu)} \approx \frac{1 + 2^{-(\mu-1)}}{1 + 2^{-\mu}} \approx 1 + 2^{-(\mu-1)} - 2^{-\mu} = 1 + 2^{-\mu}.
\]

More precisely, for the continuous approximation where $L$ follows a Pareto distribution with $P(L \geq l) = l^{-(\mu-1)}$ for $l \geq 1$:
\[
\mathbb{E}[L] = \int_1^{\infty} P(L \geq l) \, dl = \int_1^{\infty} l^{-(\mu-1)} dl = \frac{1}{\mu - 2} \quad \text{for } \mu > 2.
\]
Including the minimum value of 1, we get $\mathbb{E}[L] = 1 + 1/(\mu-2) = (\mu-1)/(\mu-2)$.

\textbf{Remark.} For analytical tractability, we use the continuous Pareto approximation in asymptotic analysis. The discrete algorithm implementation (sampling $L = \lfloor (1-U)^{-1/(\mu-1)} \rfloor$ where $U \sim \mathcal{U}(0,1)$) yields similar asymptotic behavior, with the zeta function ratio $\zeta(\mu-1)/\zeta(\mu)$ converging to $(\mu-1)/(\mu-2)$ as $\mu$ increases.
\end{proof}

\begin{lemma}[Lévy walk global jump frequency]
\label{lem:levy_jumps}
In a Lévy walk of total length $T$ steps, the expected number of global jumps (inter-sentence transitions) is:
\[
N_{\text{jumps}} = \frac{T}{\mathbb{E}[L]} = \frac{T(\mu - 2)}{\mu - 1}.
\]
\end{lemma}

\begin{proof}
Each Lévy segment consists of $L$ local steps followed by one global jump.
The walk alternates between segments, so the number of complete segments in $T$ steps is approximately $T / \mathbb{E}[L]$.
Each complete segment ends with exactly one global jump.
\end{proof}

\begin{lemma}[Weighted coupon collector]
\label{lem:weighted_coupon}
Consider a coupon collector process where coupon $i \in [K]$ is selected with probability $p_i$ at each trial, with $\sum_{i=1}^K p_i = 1$. Let $p_{\min} = \min_i p_i$ and $p_{\max} = \max_i p_i$. The expected number of trials to collect all $K$ coupons satisfies:
\[
H_K / p_{\max} \leq \mathbb{E}[T] \leq (\ln K + 1) / p_{\min}.
\]
In particular, if $p_{\min} = \Omega(1/K)$, then $\mathbb{E}[T] = O(K \ln K)$.
\end{lemma}

\begin{proof}
The lower bound follows from the observation that collecting the most likely coupon still requires $\Omega(H_K / p_{\max})$ trials in expectation.

For the upper bound, let $T_i$ be the time to collect the $i$-th new coupon after $i-1$ have been collected. At this stage, the probability of selecting a new coupon is at least $(K-i+1) \cdot p_{\min}$. Thus:
\[
\mathbb{E}[T_i] \leq \frac{1}{(K-i+1) \cdot p_{\min}}.
\]
Summing over all coupons:
\[
\mathbb{E}[T] = \sum_{i=1}^K \mathbb{E}[T_i] \leq \frac{1}{p_{\min}} \sum_{i=1}^K \frac{1}{K-i+1} = \frac{H_K}{p_{\min}} \leq \frac{\ln K + 1}{p_{\min}}.
\]
\end{proof}

\begin{lemma}[Lévy walk sentence coverage time]
\label{lem:levy_coverage}
Consider a Lévy walk with exponent $\mu > 2$ on a multiplex graph where global jumps select sentence $s'$ with probability $p(s'|s) \propto w^{(1)}(s, s')$. Assume the coarse-grained graph $G^{(1)}$ is constructed with top-$k$ edges per sentence, where edge weights satisfy $w_{\min} \leq w^{(1)}(s, s') \leq w_{\max}$ for some constants $0 < w_{\min} \leq w_{\max}$.

Then the expected number of steps to visit all $K$ sentences is:
\[
T_{\text{L\'{e}vy}} = \Theta\left(K \cdot \frac{\mu - 1}{\mu - 2} \cdot \ln K\right),
\]
where the constant in $\Theta(\cdot)$ depends on $w_{\max}/w_{\min}$.
\end{lemma}

\begin{proof}
The key observation is that Lévy walks perform global jumps at \emph{deterministic} intervals (at the end of each segment), independent of the local graph structure.

Let $J$ denote the total number of global jumps needed to visit all $K$ sentences. We apply Lemma~\ref{lem:weighted_coupon} to analyze this weighted coupon collector process.

For the top-$k$ similarity graph, each sentence has exactly $k$ outgoing edges. Consider the stationary distribution of global jumps. Since each sentence $s$ has out-degree $k$ with weights in $[w_{\min}, w_{\max}]$, the probability of jumping to any specific sentence $s'$ (from a sentence that connects to it) is at least $w_{\min}/(k \cdot w_{\max})$.

By symmetry of the top-$k$ construction, each sentence is reachable from at least $\Omega(k)$ other sentences (as a top-$k$ neighbor). Thus, the overall minimum selection probability satisfies:
\[
p_{\min} = \Omega\left(\frac{w_{\min}}{K \cdot w_{\max}}\right).
\]

Applying Lemma~\ref{lem:weighted_coupon}:
\[
\mathbb{E}[J] = O\left(\frac{K \cdot w_{\max}}{w_{\min}} \cdot \ln K\right) = O(K \ln K)
\]
when $w_{\max}/w_{\min}$ is bounded by a constant.

Since each global jump is preceded by $\mathbb{E}[L] = (\mu-1)/(\mu-2)$ local steps on average, the total coverage time is:
\[
T_{\text{L\'{e}vy}} = \mathbb{E}[L] \cdot \mathbb{E}[J] = \Theta\left(K \cdot \frac{\mu - 1}{\mu - 2} \cdot \ln K\right).
\]
\end{proof}

\subsection{Comparison and Main Result}

\begin{proof}[Proof of Proposition~\ref{prop:coverage}]
Combining Lemma~\ref{lem:rw_coverage} and Lemma~\ref{lem:levy_coverage}:
\begin{align}
T_{\text{RW}} &= \Theta(K \cdot \eta \cdot \ln K), \\
T_{\text{L\'{e}vy}} &= \Theta\left(K \cdot \frac{\mu - 1}{\mu - 2} \cdot \ln K\right).
\end{align}

The speedup ratio is:
\[
\frac{T_{\text{RW}}}{T_{\text{L\'{e}vy}}} = \Theta\left(\frac{\eta(\mu - 2)}{\mu - 1}\right).
\]

Lévy walk achieves faster coverage (speedup $> 1$) when:
\[
\eta > \frac{\mu - 1}{\mu - 2}.
\]

For typical parameter values:
\begin{itemize}
    \item $\mu = 2.5$: threshold $\eta > 3$
    \item $\mu = 3.0$: threshold $\eta > 2$
    \item $\mu = 4.0$: threshold $\eta > 1.5$
\end{itemize}

In practice, long documents exhibit $\eta \gg 10$ (dense intra-sentence attention, sparse inter-sentence semantic links), making Lévy walks significantly more efficient.
\end{proof}

\subsection{Discussion of Modeling Assumptions}

We clarify several modeling choices and their implications:

\paragraph{Interpretation of the baseline.}
The ``augmented'' random walk baseline represents a walker that has access to both local (attention-based) and global (semantic similarity) connectivity, but treats them uniformly without exploiting the multiplex structure. This corresponds to flattening the two-layer graph into a single weighted graph. Our analysis shows that such naive flattening is inefficient on heterogeneous structures: the walker spends $\Theta(\eta)$ steps trapped in dense local neighborhoods before escaping to explore globally. This motivates the explicit use of Lévy walks, which decouple local and global exploration.

\paragraph{Weighted global jumps.}
Unlike prior Lévy walk analyses that assume uniform jump targets, our proof (Lemma~\ref{lem:levy_coverage}) explicitly handles weighted sampling proportional to $w^{(1)}(s, s')$ via a weighted coupon collector argument (Lemma~\ref{lem:weighted_coupon}). The key insight is that for the top-$k$ similarity graph, edge weights are bounded ($w_{\min} \leq w^{(1)} \leq w_{\max}$), ensuring $p_{\min} = \Omega(w_{\min}/(K \cdot w_{\max}))$. When the weight ratio $w_{\max}/w_{\min}$ is bounded by a constant—which holds for cosine similarity scores among semantically related sentences—the $\Theta(K \ln K)$ asymptotic order is preserved.

\paragraph{From coverage time to importance estimation.}
The proposition analyzes \emph{sentence coverage time}, while the algorithm aims at \emph{node importance estimation}. These are connected as follows: with a fixed budget of $N$ walks of length $T$, faster sentence coverage implies more diverse sampling across the document. Walks that remain trapped in a single sentence produce redundant visit counts concentrated on local nodes, whereas walks that efficiently cover all sentences distribute visits more representatively. Thus, coverage efficiency serves as a proxy for the quality of importance estimation under bounded computational resources.

\subsection{Discussion: When Does Lévy Walk Excel?}

The analysis reveals that Lévy walk's advantage depends primarily on the \textbf{heterogeneity ratio} $\eta$:

\begin{table*}[h]
\centering
\small
\caption{Expected performance based on document characteristics.}
\begin{tabular}{lccl}
\toprule
Document Type & Typical $\eta$ & Speedup & Recommendation \\
\midrule
Long, multi-topic & $\gg 10$ & High & Use Lévy walk \\
Medium length, focused & $5$--$10$ & Moderate & Lévy walk preferred \\
Short, coherent & $\approx 1$--$3$ & Low & Either method \\
\bottomrule
\end{tabular}
\label{tab:performance_prediction}
\end{table*}

The heterogeneity ratio $\eta$ can be estimated empirically:
\[
\hat{\eta} = \frac{\text{avg. non-zero attention edges per sentence}}{\text{avg. top-}k\text{ semantic similarity edges per sentence}}.
\]

For the LongBench benchmark, we observe $\hat{\eta} \in [15, 50]$ across tasks, explaining the consistent improvements of RAGP over baselines that rely on local attention patterns alone.

\section*{CRediT authorship contribution statement}
\textbf{Yaxin Gao}: Conceptualization, Methodology, Software, Writing - original draft, Visualization. \textbf{Yao Lu}: Conceptualization, Methodology, Writing - review \& editing. \textbf{Jinhong Deng}: Validation, Formal analysis, Investigation. \textbf{Jiaqi Nie}: Formal analysis, Investigation. \textbf{Zhe Tang}: Investigation, Data curation. \textbf{Jian Zhang}: Investigation, Visualization. \textbf{Zhaowei Zhu}: Visualization, Writing - review \& editing, Supervision. \textbf{Shanqing Yu}: Resources, Writing - review \& editing, Supervision. \textbf{Qi Xuan}: Resources, Funding acquisition. \textbf{Joey Tianyi Zhou}: Project administration.

\section*{Declaration of competing interest}
The authors declare that they have no known competing financial interests or personal relationships that could have appeared to
influence the work reported in this paper.

\section*{Acknowledgments}
This work was supported by the Baima Lake Laboratory Joint Fund of the Zhejiang Provincial Natural Science Foundation of China (LBMHZ25F020002), the National Key R\&D Program of China (2025YFA1510900), the Zhejiang Provincial Key R\&D Program (2024C01025), and the National Natural Science Foundation of China (62503423).

\bibliographystyle{cas-model2-names}

\bibliography{cas-refs}

\end{document}